%% file: main.tex
\documentclass[11pt]{article}

\usepackage[utf8]{inputenc} 
\usepackage[T1]{fontenc}    
\usepackage[hidelinks]{hyperref}       
\usepackage{url}            
\usepackage{booktabs}       
\usepackage{amsfonts}       
\usepackage{nicefrac}       
\usepackage{microtype}      
\usepackage{wrapfig}
\usepackage{amsmath}
\usepackage{algorithmic}
\usepackage{array}
\usepackage{amssymb}
\usepackage{amsthm}
\usepackage{cleveref}
\usepackage{thmtools}
\usepackage[ruled,vlined,linesnumbered]{algorithm2e}
\usepackage{xcolor}
\usepackage{booktabs}
\usepackage{graphicx}
\usepackage{tikz}
\usepackage{comment}
\usepackage{pgfplots}
\usepackage{array}
\usepackage{todonotes}
\usepackage{subcaption}
\usepackage{paralist}
\usepackage[left=1.1in,right=1.1in,top=1.1in,bottom=1.1in]{geometry}

\input{formatting}

\title{Model Explanations with Differential Privacy}

\author{
  Neel Patel, Reza Shokri, Yair Zick\\
  National University of Singapore (NUS)\\
  \texttt{\{neel, reza, zick\}@comp.nus.edu.sg} \\
}

\date{}

\begin{document}

\maketitle
\input{abstract}

\input{intro}

\input{problem}
\input{basic}

\input{adaptive}
\input{dptraining}
\input{experiments}

\input{broaderimpact}

\input{conclusions}

\newpage
\bibliographystyle{plain}
\bibliography{abb,references}
\appendix
\input{appendix}

\end{document}

%% file: formatting.tex
\renewcommand{\cal}[1]{\mathcal{#1}}

\newcommand{\R}{\mathbb{R}}
\newcommand{\E}{\mathbb{E}}

\renewcommand{\E}{\mathbb{E}}
\renewcommand{\R}{\mathbb{R}}
\newcommand{\priv}{\mathit{Priv}}

\newcommand{\euler}{\mathrm{e}}
\renewcommand{\cal}[1]{\mathcal{#1}}

\newcommand{\bestPIEP}[1]{\textbf{#1}\textcolor{blue}{\textbf{$\uparrow$}}}
\newcommand{\worstPIEP}[1]{\textbf{#1}\textcolor{blue}{\textbf{$\downarrow$}}}

\newcommand{\bestLIME}[1]{\textbf{#1}\textcolor{red}{\textbf{$\uparrow$}}}
\newcommand{\worstLIME}[1]{\textbf{#1}\textcolor{red}{\textbf{$\downarrow$}}}
\newcommand{\bestMIM}[1]{\textbf{#1}\textcolor{green}{\textbf{$\uparrow$}}}
\newcommand{\worstMIM}[1]{\textbf{#1}\textcolor{green}{\textbf{$\downarrow$}}}

\newcommand{\spent}{\textit{spent}}

\newcommand{\dotappend}{\texttt{.append}}
\newcommand{\best}{\mathit{best}}
\DeclareMathOperator*{\argmax}{arg\,max}
\DeclareMathOperator*{\argmin}{arg\,min}

\newcommand{\trainingepsilon}{\hat \epsilon}
\newcommand{\trainingdelta}{\hat \delta}
\newcommand{\dist}{\mathtt{Dist}}
\newcommand\independent{\protect\mathpalette{\protect\independenT}{\perp}}
\def\independenT#1#2{\mathrel{\rlap{$#1#2$}\mkern2mu{#1#2}}}
\declaretheorem[name=Theorem,numberwithin=section]{theorem}

\declaretheorem[name=Lemma,sibling=theorem]{lemma}

\crefname{lemma}{Lemma}{Lemmas}

\theoremstyle{definition}
\theoremstyle{definition}

\SetCommentSty{mycommfont}
\SetKwFunction{DPGrad}{DPGD-Explain}
\SetKwFunction{Parameters}{Parameters-DPGD}
\SetKwFunction{flag}{Flag}

\newcommand{\coloremph}[1]{\textcolor{red!80!black}{\textit{#1}}}

%% file: abstract.tex
\begin{abstract}
Black-box machine learning models are used in critical decision-making domains, giving rise to several calls for more \coloremph{algorithmic transparency}.  The drawback is that model explanations can leak information about the training data and the explanation data used to generate them, thus undermining \coloremph{data privacy}.  To address this issue, we propose differentially private algorithms to construct feature-based model explanations.  We design an adaptive differentially private gradient descent algorithm, that finds the minimal privacy budget required to produce accurate explanations.  It reduces the overall privacy loss on explanation data, by adaptively reusing past differentially private explanations.  It also amplifies the privacy guarantees with respect to the training data.  We evaluate the implications of differentially private models and our privacy mechanisms on the quality of model explanations. 
\end{abstract}

%% file: intro.tex
\section{Introduction}\label{sec:intro}
Machine learning models are currently applied in a variety of high-stakes domains, e.g. providing predictive healthcare analytics, assessing insurance policies, and making credit decisions. These domains require high prediction accuracy over high-dimensional data; as a result, they adopt increasingly complex architectures, making them harder to interpret.  What makes the problem even more challenging is that these models are usually used as \coloremph{black-box algorithms}: they only output a decision, without providing detailed information on its intermediate computations.  Growing mistrust of black-box algorithms in high-stakes domains resulted in a widespread call for \coloremph{algorithmic transparency} \cite{guidotti2019survey, lipton2016interpretability}. This term broadly refers to methods that offer additional information about the underlying algorithmic decision-making process.  In this work, we focus on model-agnostic feature based explanations as in \cite{baehrens2010explain, datta2015influence, datta2016algorithmic, lundberg2017unified, ribeiro2016should, sliwinski2019axiomatic}.  Model-agnostic methods often use a dataset labeled by the black-box model to generate model explanations, referred to as the \coloremph{explanation dataset}. This dataset is often just a subset of the training data, or sampled from the same distribution. 

The problem is that offering additional information can result in significant \coloremph{data privacy risks}, by leaking sensitive information about the underlying explanation and training data, which can be exploited by inference attacks~\cite{shokri2019privacy}.  Some feature-based model explanations directly depend on model parameters~\cite{ancona2018towards, bach2015pixel, shrikumar2017learning, simonyan2013deep, sundararajan2017axiomatic}, which even further increases their privacy vulnerabilities with respect to white-box inference attacks~\cite{nasr2019comprehensive}.  As a related issue, model explanations can also magnify the risks of model reconstruction attacks~\cite{milli2018model}, which are however not the focus of this paper.  Thus far, there has been little work on modeling the relation between algorithmic \coloremph{transparency and privacy}~\cite{datta2016algorithmic, harder2020interpretable, shokri2019privacy}, and on designing \coloremph{model explanations that protect data privacy}.  The destructive impact of randomized privacy mechanisms on the fidelity of model explanations is also not studied. 

\subsection{ Our Contributions} 

We propose provably sound, model-agnostic, and differentially private algorithms for computing model explanations that protect the sensitive information in the \coloremph{explanation} and \coloremph{training datasets}.  Our explanations are sound in the sense that they are provably similar to some standard model explanations, such as LIME \cite{ribeiro2016should}. Our methodology can be adapted to any explanation method that relies on generating accurate local models around the data point one tries to explain.

We first design a baseline interactive mechanism that outputs feature-based model explanations, while guaranteeing differential privacy for each query. Building upon this baseline, we design our main algorithm: an \coloremph{adaptive differential privacy mechanism for model explanations}.  The main challenge, which we resolve in this approach, is to optimize the composed privacy loss of the model explanation over all queries, while maintaining low explanation error. The adaptive algorithm utilizes previously released DP explanations effectively, significantly reducing privacy spending on new model explanation queries.  
We achieve this by selecting a better initialization point for the our underlying gradient descent algorithm using past queries.  
We also improve upon bounds in~\cite{shamir2013stochastic} on the convergence of the DP gradient descent method (under a minor assumption), offering faster initialization-dependent convergence rates.  
We further show that when we approach the optimal point, the DP gradient descent algorithm oscillates around the optimal point, spending a significant chunk of the privacy budget while obtaining only a negligible expected utility.  
This insight leads to an \coloremph{enhanced adaptive algorithm} which offers far better privacy savings at only a minor loss in accuracy.  Finally, we propose a \coloremph{non-interactive} differential privacy algorithm which generates explanations without spending any further privacy budget.

Our algorithms bound the (differential) privacy loss with respect to explanation data. 
Privacy protection of the training data is normally done during model training, or via differentially private prediction algorithms.  However, we show that our mechanism, which uses the underlying model as a black-box, can also moderately amplify the privacy protection of the training data.  
As mentioned previously, protecting the privacy of sensitive data must be done in tandem with \coloremph{preserving explanation quality}: this is crucial as the randomness in differential privacy mechanisms might reduce explanation fidelity.  
Ensuring that our mechanisms preserve explanation quality requires careful analysis of the underlying gradient descent mechanism: we determine the number of iterations required to minimize the privacy budget spending required for an explanation query, hence maximizing the overall explanation quality.  When protecting data using differentially private models, the randomness in the decision boundaries of differentially private models decreases the convergence speed of our approximation algorithms, which results in spending more privacy budget. This is demonstrated in our empirical results: stronger privacy requirements for the underlying model degrade explanation quality.

\subsection{Related Work} 

Model explanations are vulnerable to membership inference attacks that can infer if a data point is part of the model's training set~\cite{shokri2019privacy}.  A recent work~\cite{harder2020interpretable} studies the construction of differentially private and interpretable \coloremph{predictions} --- rather than explanations --- for neural networks.  But, it does not offer any bounds on the utility loss due to privacy.  An earlier work presents QII~\cite{datta2016algorithmic} that designs an interactive differentially private mechanism to protect the explanation dataset.  However, repeated queries result in an unacceptable cumulative information leakage and a loose privacy guarantee.  It also does not provide any privacy analysis for the training dataset.

%% file: problem.tex
\section{Problem Statement}\label{sec:explanations-privacy}

\begin{table}
  \caption{\small Table of Notations}
  \label{tab:notations}
  \begin{tabular}{rl}
    \toprule
    Notation & Description \\
    \midrule
    $\cal X$ & explanation dataset\\
    $\cal T$ & training dataset\\
    $m$ & size of the dataset $\cal X$\\
    $n$ & dimension of the dataset $\cal X$\\
    $f(\cdot )$ & black-box model\\
    $(\trainingepsilon, \trainingdelta)$ & privacy parameters differentially private model\\
    $(\epsilon, \delta)$ & privacy parameters for explanation for explanation dataset\\
    $\phi^\priv (\cdot)$ & private explanation \\
    $\cal C_{p,k}$ & convex set of possible explanations $\{\phi: \|\phi\|_p \leq k \}$\\
    $\alpha:\R^+ \rightarrow \R^+$ & weight function\\
    $\cal L(\cdot )$ & loss function\\
    $\phi^*(\cdot)$ & optimal solution of the loss function $\cal L(\cdot)$\\
    $\cal F(c,\vec z)$ & class of weight functions with bounded sensitivity of $c$\\
    $\cal E(\phi^\priv)$ & utility loss of the private explanation $\phi^\priv$\\
    $(\epsilon_{\textit{min}}, \delta_{\textit{min}})$ & minimum privacy requirement for each query\\ 
    $\epsilon_{ite}$ & privacy spending at each gradient descent iteration\\
    $\sigma_{\min}$ & variance required by the Gaussian mechanism \\
    $\cal H$ & history consist of explained queries in past\\
    $\cal X'$ & proxy explanation dataset for non-interactive phase\\
    $\gamma$ & amplification of the training privacy\\
    \bottomrule
  \end{tabular}
\end{table}

Consider a black-box decision-making machine learning model ${f:\R^n \to \cal R}$ trained on the \coloremph{ training dataset} $\mathcal T$. The model returns the predicted label, and provides no more information about its decision.  Our goal is to generate a \coloremph{model explanation} $\phi(\vec z, \cal X,f(\cal X))$, whose input is a \coloremph{point of interest} $\vec z$, an \coloremph{explanation dataset} ${\cal X = \{\vec x_1,\dots, \vec x_m\} \subset \R^n}$ (used to generate the explanation), and the output of $f$ on $\cal X$.  Model explanation algorithms can also use additional information: this can be a prior over the data \cite{baehrens2010explain}, access to model queries \cite{adler2018auditing, datta2016algorithmic, ribeiro2016should}, knowledge about the model hypothesis class \cite{ancona2017unified, shrikumar2017learning, sundararajan2017axiomatic}, or access to the source code \cite{datta2017use}. Table~\ref{tab:notations} summarizes our notations.
We focus on \coloremph{feature-based} model explanations: $\phi(\vec z)$ is a vector in $\R^n$, where $\phi_i(\vec z)$ measures the effect that the $i$-th feature has on the predicted label $f(\vec z)$.
 
Existing explanation algorithms assume some degree of access to data labeled by $f$.  Access to the relevant data for generating explanations is crucial: when having only black-box access to the model, information about its behavior over data is essentially the only available source one can use to generate explanations. 
In addition, it is crucial that $\cal X$ is sampled from the same distribution as the data used to train $f$; otherwise, one runs the risk of generating explanations that are inappropriate for the model's actual behavior on the data that it was given (see discussion in \cite{datta2016algorithmic}).  
Indeed, using the training set to generate explanations is standard practice~\cite{guidotti2019survey}.

\subsection{Model Explanations}

Feature based explanations are often local model approximations in a region around the point of interest $\vec z$ \cite{ribeiro2016should,sliwinski2019axiomatic,simonyan2013deep,smilkov2017smoothgrad}.  Our objective is to find a linear function $\phi$, centered at a point of interest $\vec z$, that minimizes the \coloremph{local empirical model error} over the explanation dataset $\cal X$.  
We define the \coloremph{empirical local loss} ${\cal L(\phi,\vec z,\cal X, f)}$ of a linear function $\phi$, over $\cal X$ labeled by $f$, as 
\begin{align}
	\cal L(\phi,\vec z,\cal X, f) \triangleq  \frac{1}{|\cal X|} \sum_{\vec x \in \cal X} \alpha(\| \vec x - \vec z \|) ( \phi^\top \cdot (\vec x - \vec z) - f(\vec x))^2, 
	\label{eq:MSE}
\end{align}
where, $\alpha:\R_+ \to \R_+$ is a weight function.  
We find a local explanation within some region $\cal C$ that minimizes $\cal L(\phi,\vec z,\cal X, f)$.  We assume $\cal C$ to be a bounded convex set; more concretely, we set $\cal C$ to ${{\cal C}_{p,k}=\{\phi: \| \phi \|_p \leq k\}}$. An \coloremph{optimal} model explanation is thus one that minimizes loss around the point of interest.
\begin{align}
         \phi^*(\vec z) = \argmin_{\vec \phi \in \cal C}  \cal L(\phi,\vec z,\cal X, f) 
         \label{eq:LIME}
\end{align}
The weight function $\alpha$ is a decreasing function in $\| \vec z - \vec x \|$. In other words, $\phi$ is a \coloremph{local explanation}, rewarded for correctly classifying points closer to $\vec z$. 
For example, in the LIME framework, $\alpha$ takes a value of $0$ for any points $\vec x \in \cal X$ that are more than a certain distance away from the point of interest $\vec z$.  
Our framework generalizes LIME and other local optimization-based model explanations.

\subsection{Evaluation Metrics}

We assume an \coloremph{adversary} queries the model $f$ with a sequence of data points, requesting explanations for the model's decisions over them.   Given the obtained explanations, the adversary can reconstruct sensitive information about the explanation data $\cal X$, and the training data $\cal T$.
Our objective is to design a differentially private algorithm that can provably \coloremph{protect private data against inference attacks}, and provides accurate explanations.  
We evaluate our explanation models based on two key performance metrics: their differential \coloremph{privacy loss}, and their \coloremph{utility loss} in comparison to the \coloremph{optimal explanation}, as described in Eq.~\eqref{eq:LIME}. 
 
Following the definition of \coloremph{differential privacy} (DP)~\cite{dwork2014algorithmic}, a model explanation $\phi(\cdot)$ is $(\epsilon,\delta)$-differentially private if for any sequence of queries (points of interest) $\vec z_1, \vec z_2, \cdots, \vec z_k$, for any two neighboring explanation datasets $\cal X$ and $\cal X'$ --- i.e. datasets that differ by the omission of a single point ---  and any explanations $S_i \subseteq \R^n$, we have
\begin{align}
&\Pr[\phi_1 \in S_1, \phi_2 \in S_2, \cdots, \phi_k \in S_k]   
\leq \euler^{\epsilon} \cdot \Pr[\phi'_1 \in S_1, \phi'_2 \in S_2, \cdots, \phi'_k \in S_k] + \delta,  \label{eq:DP-explanation}    
\end{align}
where, $\phi_i = \phi(\vec z_i, \cal X, f(\cal X))$, and $\phi'_i = \phi(\vec z_i, \cal X', f(\cal X'))$. We similarly define guarantees for two neighboring training datasets $\mathcal T$ and $\mathcal T'$. This guarantees ensures that the adversary cannot accumulate enough information from their queries to reconstruct data records in the explanation set or taining set, up to the protection level determined by privacy loss parameters for both datasets $\epsilon$ and $\delta$ respectively.
 
Our main focus is to protect the sensitive information about the explanation dataset $\mathcal X$, in Section~\ref{sec:privacy-training-dataset}, we show that our private explanations are differentially private w.r.t. \coloremph{training dataset} as well.   
 
The \coloremph{quality} of a model explanation is measured in terms the difference between its (expected) loss and the loss of the optimal explanation which is referred as approximation loss of explanation:
\begin{equation}
\cal E(\phi) = \E\left[ \cal L(\phi,\vec z, \cal X, f) \right] - \cal L(\phi^*(\vec z), \vec z, \cal X, f), \label{eq:utility}
\end{equation}

%% file: basic.tex
\input{alg-basic-interactive}

In the following sections, we describe a variety of mechanisms computing differentially private model explanations. We start with a simple non-adaptive mechanism in Section \ref{sec:PIEP}, using it as a foundation for more complex adaptive mechanisms (Section \ref{sec:APIEP}). Finally, we show that our mechanisms offer privacy guarantees with respect to the training data as well (Section \ref{sec:privacy-training-dataset}).
\section{A Basic DP Mechanism for Model Explanation}
\label{sec:PIEP}

We begin our exploration of differentially private model explanations with a simple approach: approximating each local model explanation in a differentially private manner.  Algorithm~\ref{Algorithm-1} presents a gradient descent algorithm to solve \eqref{eq:LIME} while satisfying differential privacy as in \eqref{eq:DP-explanation}.

We compute a differentially private model explanation $\phi^{\priv}$ by optimizing \eqref{eq:LIME} in the convex set $\cal C_{2,1}$.  The gradient descent algorithm (the \DPGrad{} procedure) uses the Gaussian mechanism in each iteration to guarantee differential privacy with respect to the explanation dataset. This is an interactive DP mechanism: for each query, we refer to the (explanation) dataset for computing the explanations.

The input to our proposed algorithm is the point of interest $\vec z$, and the explanation dataset $\mathcal{X}$ labeled by the function $f$.  It finds the solution in a feasible convex set $\mathcal{C}$ (where we use $\cal C_{2,1}$), for a local loss function $\mathcal{L}(\cdot)$.  Its output is a differentially private model explanation $\phi^\priv(\vec z,\cal X,\cal C,\cal L)$, as we prove later in this section. 

\subsection{Sensitivity of the Loss Function}

In order to design the differential privacy mechanism, with bounded privacy loss, we first need to bound the global sensitivity of the gradient of the loss function $\nabla \mathcal{L}(\cdot)$. 
The global sensitivity of $\nabla\mathcal{L}(\cdot)$ could be unbounded for many choices of $\alpha(\|\vec x -\vec z\|)$; for example, when $\alpha(\|\vec x - \vec z\|) = \frac{1}{\|\vec x - \vec z \|_2}$, this weight function is non-increasing which is a required property for local model approximation. However, the sensitivity of $\nabla\mathcal{L}(\cdot)$ for this $\alpha(\cdot)$ is unbounded. Thus, we need to choose from a family of weight functions that result in a bounded sensitivity for $\cal \nabla L(\cdot)$, required by our DP mechanisms. 
 
In Lemma~\ref{lem:dataset-diff} (whose proof is in the appendix), we characterize set of weight functions $\alpha(\| \vec x - \vec z \|_2) $ such that the sensitivity of $\nabla \mathcal{L}(\cdot)$ is bounded.
\begin{restatable}[Conditions for Bounded Sensitivity for $\cal \nabla L(\cdot)$]{lemma}{lemmautilitybound}\label{lem:dataset-diff}
	For all explanation dataset $\cal X$ of size $m$, and all its neighboring datasets $\cal X'$, points of interest $\vec z$, and $\phi \in \cal C$, the sensitivity of the gradient of loss is bounded,
	$\|\nabla \mathcal{L}(\phi,\mathcal{X})  - \nabla \mathcal{L}(\phi,\mathcal{X'}) \|_2 \leq  (\frac{c}{m}\big )$, iff $\alpha(\|\vec x - \vec z\|) \le \frac{c}{2\|\vec{x}- \vec{z}\|_2(\|\vec{x}- \vec{z}\|_2 + 1)}$ for every $\vec x, \vec z \in \R^n$.  
\end{restatable}
\begin{proof}
See Appendix~\ref{appendix:prrofsensitivity}
\end{proof}

Lemma~\ref{lem:dataset-diff} characterizes all weight functions $\alpha(\cdot)$ for which the sensitivity of the gradient $\nabla \cal L (\phi, \cal X)$ is bounded.  
We define the family of desirable weight functions as:
\begin{align}
\cal F (c,\vec z) = \nonumber\left\{ \alpha(\cdot): \begin{array}{l}\alpha(\cdot) \textit{ is non-increasing and}\\ \forall \vec x \in \mathbb{R}^n , \alpha(\|\vec x - \vec z\|) \leq \frac{c}{2\|\vec x- \vec z \|(\|\vec x - \vec z \| + 1)}\end{array} \right\}.
\end{align} 

Note that $\cal F(c,\vec z)$ is non-empty as $\frac{c}{2\|\vec x-\vec z\|(1+\|\vec x -\vec z\|)}$ belongs to the family.The choice of weight function governs how ``local'' the model approximation is around a point of interest.  For example, $\alpha(\|\vec x - \vec z \|)=\euler^{-\|\vec x - \vec z \|^2}$ is in $\cal F(1,\vec z)$ which exponentially decreases the weight of the data points according to their distance to the point of interest $\vec z$. This weight function is extremely local. Thus, it is useful when the explanation dataset $\cal X$ is highly dense in different regions.  

\subsection{Privacy and Utility Guarantees}

We bound the privacy and utility loss of model explanations produced by Algorithm~\ref{Algorithm-1}.  The data dependent quantity, in each iteration of Algorithm~\ref{Algorithm-1}, is $\nabla \cal L(\phi,\cal X)$. 
Given Lemma~\ref{lem:dataset-diff}, we use the Gaussian mechanism to ensure that the update in each round is differentially private.  Using the differential privacy composition theorem~\cite[Theorem-4.3]{kairouz2017composition}, the output of algorithm~\ref{Algorithm-1} is $(\epsilon,\delta)-$ differentially private.  We show that the differentially private mechanism for model explanation offers good utility loss guarantees.

\begin{restatable}[Bounded Utility Loss]{theorem}{PIEPutilitybound}\label{thm:uti_bound}
Algorithm~\ref{Algorithm-1} is $(\epsilon,\delta)$ differentially private. Moreover, for given privacy parameters $(\epsilon,\delta)$, if \DPGrad{} runs for $T \le \min \left( \frac{m^2\epsilon^2}{32n\log^2(\euler + 1/\delta)},\frac{1}{\delta} \right)$ iterations and outputs an explanation $\phi^{\priv}$, then
		$\mathcal{E}(\phi^{\priv}) \in \mathcal{O}\left( \frac{ \log T}{\sqrt{T}} \right)$. For $T = \frac{m^2\epsilon^2}{32n\log^2(\euler + 1/\delta)}$,  
		${\mathcal{E}(\phi^{\priv}) \in \mathcal{O}\left(\frac{\sqrt {n}\log m }{m\epsilon}\right)}.$
\end{restatable}

The bound in Theorem \ref{thm:uti_bound} is reasonable when the number of features is smaller than the size of the explanation dataset.  
Let $\phi^\priv$ be the output of the mechanism described in Theorem \ref{thm:uti_bound}; while $\phi^\priv$ is differentially private in itself, its privacy loss accumulates as we make additional model explanation queries.
Following the DP composition theorem \cite{kairouz2017composition}, $k$ independent $(\epsilon,\delta)$-private queries result in an $(\tilde{\epsilon},k\delta+\tilde{\delta})$-differentially private output for any $\tilde{\delta} \in (0,1)$ and $\tilde{\epsilon} \in \cal O\left(\epsilon\times( k\euler^{\epsilon} +\sqrt{k\log\left( \euler + \epsilon\sqrt{k}/\tilde{\delta} \right) }\right)$.

%% file: alg-basic-interactive.tex
\begin{algorithm}[!t]
	\SetAlgoLined
	\LinesNumbered
	\SetKwProg{proc}{Procedure}{}{}
	\let\oldnl\nl
	\newcommand{\nonl}{\renewcommand{\nl}{\let\nl\oldnl}}
	\SetNlSty{}{}{:}
	
	\small
	
	\nonl \textbf{Input: } POI $\vec z \in \R^n$, explanation dataset $\cal X$, DP parameters $\epsilon>0, \delta>0$, weight function $\alpha(\| \vec x - \vec z \|) \in \cal F(1,\vec z)$, learning rate function $\eta(t) \in \R_+$, and the number of GD iterations $T$.\\[5pt]
	
	\nonl \textbf{Output: } $\phi^{\priv}$\\[5pt]
	
	Set the parameter for the Gaussian mechanism $\sigma \gets \frac{1}{ m \epsilon } \sqrt{16T\log \big(\euler + \frac{\sqrt T\epsilon}{\delta}\big)\log \frac{T}{\delta}}$\;
	
	Initialize $\phi$ with an arbitrary vector in $\mathcal{C}_{2,1}$\;
	
	\Return{\DPGrad{$\phi,\sigma,T$}}\;
	
	\proc{\DPGrad{$\phi, \sigma, T$}}{
		
		$\phi^{\{1\}} \gets \phi$
		
		\For{$t = 1, \dots, T-1$}{
			$\xi_t \gets \left(\phi^{\{t\}} - \eta(t)\big[\nabla \mathcal{L}(\phi,\mathcal{X}) + \mathcal{N}(0,\sigma^2\mathrm{I})\big]\right)$
			
			$\phi^{\{t+1\}} \gets \argmin_{\phi \in \cal C_{2,1}} \| \phi - \xi_t \|$
		}
		
		\KwRet $\phi^{\{T\}}$	
	}
	\caption{\small Interactive DP Model Explanation}\label{Algorithm-1}
\end{algorithm}

%% file: adaptive.tex
\section{Adaptive Differential Privacy Mechanisms}\label{sec:APIEP}

The explanation algorithm described in Section \ref{sec:PIEP} naively computes an explanation for a queried data point: it does not utilize any {\em previously released differentially private information}. 
When Algorithm~\ref{Algorithm-1} computes a model explanation, it releases information about the underlying black-box algorithm in a differentially private manner. 
This information can be used to generate an explanation for new queries more economically, resulting in less privacy spending. 
How can we utilize past information efficiently?     

Let us define this problem more formally. Let $\vec z_1,\dots, \vec z_h$ be the sequence of queries previously explained in a differentially private manner, and their corresponding explanations: $\phi^\priv(\vec z_1),\dots, \phi^\priv(\vec z_h)$. 
When we receive a new query $\vec z_{h+1}$, our objective is to extract information from $(\vec z_1,\phi^\priv(\vec z_1)),\dots ,(\vec z_h, \phi^\priv(\vec z_h))$ such that the new query $\vec z_{h+1}$ requires less privacy budget, without compromising its explanation quality. 

If the underlying black-box model behaves in a consistent manner within a local region, then model explanations in that local region should be consistent: this enables us to exploit previous queries from the same local region to explain the current query while spending less of the privacy budget. 
Moreover, if we can ensure that the \DPGrad{} procedure converges faster by using past information, then also we can reduce the resultant privacy spending for the current query. 

The above insights are used in our {\em Adaptive Private Interactive Explanation Protocol} (Algorithm \ref{algo:APIEP}): when computing a model explanation for a new query, it is safe to use previously released information, as it was computed in a differentially private manner. 
We use past information in two ways. Firstly, similar datapoints should have similar explanations: if $\vec z$ and $\vec v$ are very close to each other in the dataspace then the model explanations for $\vec v$ and $\vec z$ ($\phi^{\priv}(\vec v)$ and $\phi^{\priv}(\vec z)$) should be similar. 
This observation allows us to save time --- and privacy budget --- when computing new differentially private model explanations. 

Secondly, instead of starting the gradient descent process in Algorithm \ref{Algorithm-1} from an arbitrary point, we utilize past information to find an approximately optimal starting, resulting in faster convergence and reduced privacy spending. 
Ideally, we want to initialize the gradient descent process at a point as close to an optimal point as possible; however, this selection process should itself be privacy-preserving. Thus, the budget spent on finding the initialization point should not exceed the savings obtained by faster convergence. 

Our explanation algorithm optimizes Equation~\ref{eq:LIME}, a convex function that has a unique minimum. 
This fact offers several indicators for a ``good'' starting point. 
Ideally, given the history $(\vec z_1,\phi^\priv(\vec z_1)),\dots (\vec z_h, \phi^\priv(\vec z_h))$, we want to select $\phi^\priv(\vec z_j)$ minimizing ${\|\phi^\priv(\vec z_j) - \phi^*(\vec z_{h+1}) \|}$. 
However, a tight bound on the sensitivity of $\min_{j} \|\phi^\priv(\vec z_j) - \phi^*(\vec z_{h+1}) \|$ is difficult to obtain, as $\phi^*(z_{h+1})$ admits no closed-form solution; thus, searching for an initialization point results in a noisier selection process, requiring more undesirable privacy spending.
 
As an alternative, our adaptive differentially private algorithm uses a greedy approach: it searches for ${\argmin_j\| \nabla \cal L (\vec z_{h+1},\phi^\priv(\vec z_j))\|}$ which is a good indicator of one's proximity to the optimum of the convex function $\cal L(\cdot)$. Moreover, the sensitivity of $\|\nabla \cal L(\vec z, \phi)\|$ is bounded by $\cal O(\frac{1}{m})$ (as per Lemma \ref{lem:dataset-diff}), which allows us to efficiently search for the optimal point while incurring minor privacy spending.   

\subsection{Finding Similar Points from Prior Explanation Queries}
Consider the following weight function, defined for $c>0$.
\begin{align}
\alpha(\|\vec x - \vec z \|) =
\begin{cases}
1, & \textit{if }\|\vec x - \vec z \|\leq \frac{\sqrt{2c+1}-1}{2}\  \\
\frac{c}{2\|\vec x - \vec z \|(1 + \|\vec x - \vec z \|)},  & \textit{else}
\end{cases} \label{eqn:alpha}
\end{align}

This weight function $\alpha$ assigns equal weight to all points $\vec x \in \cal X$ that are close to the point of interest $\vec z$, and decreases the weight for points further from $\vec z$ quadratically in their distance.  
This weight function~\eqref{eqn:alpha} is bounded by $1$, and belongs to $\cal F(c,\vec z)$.  Moreover, this weight function is ``stable'': it preserves the consistency of the local explanation in a small region. 
We formalize this property in Lemma~\ref{lem:alpha_closness}, and use the stable weight function in Equation \eqref{eqn:alpha} when running Algorithm \ref{algo:APIEP}.

In Lemma~\ref{lem:alpha_closness} we show that the stable weight function $\alpha(\cdot)$ evaluates two close queries $\vec v$ and $\vec z$ from the same region similarly. 
It shows that, if for two explanation queries $\vec v, \vec z$ satisfy $\|\vec v - \vec z\|<d$, then the ratio of their weight functions is $1+\cal O(d)$. 
This further shows that as $d$ decreases, the stable weight function assigns similar weights to neighboring datapoints. 
This property is highly desirable when the underlying black-box model is consistent in local regions: it implies that the behavior of our explanation algorithm is also consistent in local regions. For the sake of exposition, we again write 
$r = \frac{\sqrt{2c+1}-1}{2}$, and assume that $c$ (and in particular $r$) is a constant.

\begin{restatable}[Stability of the $\alpha(\cdot)$ Function]{lemma}{lemmaalphacloseness}\label{lem:alpha_closness}

If $\|\vec v - \vec z\|\leq d$ then for all $\vec x \in \cal X$: 
$$\frac{\alpha(\|\vec x - \vec v\|)}{\alpha(\|\vec x - \vec z\|)} \leq 1 + \max \left( \frac{(d^2 + 2rd)}{r^2}, 2(d^2+2rd+d)\right).$$ 
Furthermore, if $d\in o(r)$ then  $\frac{\alpha(\|\vec x - \vec v\|)}{\alpha(\|\vec x - \vec z\|)}  < 1+\cal O(d)$.

\end{restatable}

\begin{proof}
See Appendix~\ref{appendix:proof_alpha_closeness}.
\end{proof}

Lemma~\ref{lem:alpha_closness} bounds the difference in $\alpha$ when $\|\vec v - \vec z\| \leq d \ll r$ for all $\vec x \in \cal X$. 
Theorem~\ref{thm:adpative_closness} shows that for the stable weight function, the queries $\vec v, \vec z$ will have similar explanations if $\|\vec v - \vec z\|\leq d\leq r$. 
If $\phi^*(\vec z)$ is the optimal solution of Equation~\ref{eq:LIME} for $\vec z$, then Equation~\eqref{eqn:similar_loss} shows that $\cal L(\phi^*(\vec z), \vec v, \cal X) \approx \cal L(\phi^*(\vec z), \vec z, \cal X)$ when $d$ is small. 

\begin{restatable}{theorem}{thmadaptivecloseness}{\normalfont(Consistent Explanations in Small Regions)}\label{thm:adpative_closness}
For $\alpha(\cdot)$ described in~\eqref{eqn:alpha}, if we have ${\phi^*(\vec z) \in \argmin_{\phi \in \cal C} \cal L(\phi,\vec z, \cal X)}$, and ${\phi^*(\vec v) = \argmin_{\phi \in \cal C} \cal L(\phi,\vec v, \cal X)}$ and ${\|\vec x - \vec z\|\leq d\ll r}$, then 
\begin{equation*}
    \cal L(\phi^*(\vec z), \vec v,\cal X) - \cal L(\phi^*(\vec v), \vec v,\cal X) < 8d + \cal O(d^2).
\end{equation*}
Moreover, $|\cal L(\phi, \vec v,\cal X) - \cal L(\phi, \vec z, \cal X)| \in \cal O(d^2)$ for all $\phi \in \cal C$
\end{restatable}

\begin{proof}
See Appendix~\ref{appendix:proofofclosness}
\end{proof}

We use Theorem~\ref{thm:adpative_closness} to save our privacy budget when we receive two explanation queries from the same neighborhood.   
Suppose that the current queried point $\vec v$ satisfies $\|\vec v - \vec z \|\leq d < r$, and $\vec z$ is already explained in a differentially private manner. We can utilize the explanation for $\vec z$ in order to compute an explanation for $\vec v$ without spending any privacy budget with little approximation loss. However, this result only allows us to save the privacy budget if the query $\vec v$ is in the region of some previously explained query $\vec z$.

\subsection{Reusing Prior Explanations for a Better Optimization}

Suppose that $\vec v$ is not within a small region of $\vec z$, but the model exhibits similar behavior around $\vec z$ and $\vec v$, i.e. $\|\nabla \cal L(\vec \phi^\priv(\vec z),\vec v,\cal X)\|_2$ is sufficiently small; in this case, we can use the explanation of $\vec z$ as an initialization point $\vec \phi^{\{0\}}$ for the gradient descent procedure used to generate an explanation for $\vec v$. 

Theorem~\ref{thm:small_grad_priv} shows that if $\|\nabla \cal L(\vec \phi^\priv(\vec z),\vec v)\|_2$ is sufficiently small, then the gradient descent method initialized with $\phi^{\{0\}} = \phi^\priv(\vec z)$ converges significantly faster when computing an explanation for $\vec v$. 
We also bound the number of iterations required in order to achieve same approximation error for our explanation compared to a random starting point, as a function of $\|\nabla \cal L (\phi^{\{0\}},\vec v,\cal X) \|$. 

Before we prove the Theorem~\ref{thm:small_grad_priv}, we present a technical Lemma~\ref{lem:small_grad} in which we obtain the initialization dependent upper bound on the approximation error.
Lemma~\ref{lem:small_grad} shows that the differentially private gradient descent algorithm converges faster when given a better initialization point, under minor assumptions. 
In particular, we show that for $T\leq \euler^{1/n\sigma^2}$, $\cal E(\phi^{\{T\}},\vec v) = \frac{q\log T}{\sqrt{T}}$ where $q\leq 1$, and it decreases linearly in $\|\cal L (\phi^{\{0\}})\|$. Equation~\eqref{eqn:learning_rate} provides the learning rate for the differentially private gradient descent algorithm. We note that in general the $\frac{q'\log T}{\sqrt T}$ bound holds for differentially private gradient descent when $\sqrt n\sigma\leq 1$ (see Theorem~\ref{thm:uti_bound}) where $q'>1$ therefore $T\leq \euler^{1/n\sigma^2}$ is not an unreasonable assumption.

Recall that \DPGrad{} in Algorithm~\ref{Algorithm-1} uses the following update rule:
\begin{align}
\xi_t =\left(\phi^{\{t\}} - \eta(t)\big[\nabla \mathcal{L}(\phi,\mathcal{X}) + \mathcal{N}(0,\sigma^2\mathrm{I})\big]\right), \quad \phi^{\{t+1\}} =\argmin_{\phi \in \cal C_{2,1}} \| \phi - \xi_t \|\label{eq:grad-desc-update-rule}
\end{align}

\begin{restatable}[Initialization Dependent Upper-Bound]{lemma}{lemmasmallgrad}\label{lem:small_grad}
If $\|\nabla\mathcal{L}(\phi^{\{0\}},\vec v,\mathcal{X})\| = \beta$, then for some $c$, and $\eta(t) = \frac{c}{\sqrt{t}}$;
\begin{equation}
    \cal E(\phi^{\{T\}},\vec x) \in \cal O  \left( \left( (\sigma^2n + \beta^2 )^\frac{1}{2} + (n\sigma^2 \log T)^\frac{1}{4}  \right) \frac{\log T}{\sqrt{T}} \right) \label{eqn:initialization_dependent_bound}
\end{equation}
\end{restatable}
\begin{proof}
 See Appendix~\ref{appendix:proofofintUB}
\end{proof}

\begin{restatable}[Number of Iterations]{theorem}{coronumberofiteration}\label{thm:small_grad_priv}
Let $\phi^{\{0\}} = \phi^\priv(\vec z)$ be the initialization point given to the \DPGrad{} procedure, with $\|\nabla\mathcal{L}(\phi^{\{0\}},\vec v,\mathcal{X})\| = \beta$.  Given the update rule in Equation \eqref{eq:grad-desc-update-rule}, if $\max(\sqrt{n} \sigma, \beta) \leq \frac{1}{(\log T)^a}$ for $a>\frac{1}{2}$, then for $T' = \max (\sqrt{n}\sigma,\beta)^{1-\frac{1}{2a}} T$: 
$\cal E (\phi^{\{T\}},\vec v)\in \cal O\left( \frac{\log T}{ \sqrt{T} }\right) .$
\end{restatable}
\begin{proof}
See Appendix~\ref{appendix:fasterconvergence}
\end{proof}

Our Adaptive approach uses Theorem~\ref{thm:small_grad_priv} to reduce privacy spending, by accelerating the optimization process. If Algorithm \ref{algo:APIEP} has already explained some query $\vec z$ with a sufficiently small $\|\nabla\mathcal{L}(\phi^\priv(\vec z),\vec v\| = \beta$ then by initializing the optimization to $\phi^{\{0\}}=\phi^\priv(\vec z)$, we require to spend some ($\beta$ dependent) exponent of $\beta$ factor of the total privacy required for the given approximation loss. 
Moreover, this exponent is a non-decreasing function of $\beta$, which ensures that better initialization points require fewer iterations to achieve a better approximation. 

\input{alg-parameter-selection.tex}

\subsection{An Adaptive Explanation Protocol} \label{sec:adaptive-algorithm}

The key idea of our algorithm is that once it releases enough information about the underlying explanation dataset, it utilizes this information to explain new queries with less privacy spending using Theorems~\ref{thm:adpative_closness} and~\ref{thm:small_grad_priv}. 

Our mechanism (Algorithm \ref{algo:APIEP}) takes a total privacy budget $(\epsilon,\delta)$, a minimum privacy required for each query $(\epsilon_{\textit{min}},\delta_{\textit{min}})$ (the maximum information leakage allowed per query), an explanation dataset $\cal X$, a maximal number of iterations $T$ (which also defines the explanation quality requirement), and adaptively generates differentially private model explanations for a string of queries $\vec z_1,\dots $ until it exhausts its entire privacy budget $(\epsilon, \delta)$.
 
Given the minimum privacy loss parameters $\epsilon_{\textit{min}},\delta_{\textit{min}}$, we have a lower bound on the variance of Gaussian noise $\sigma_{\textit{min}}$ (Gaussian Mechanism \cite{dwork2006differential}), added at each iteration, where $\epsilon_{ite}$ is computed using the composition theorem. 
This is the minimum variance required to achieve the resultant $\delta$ parameter with at least $\delta_{\textit{min}}$ spent per query.
 
At each query $\vec z_h$, Algorithm \ref{algo:APIEP} first inspects whether it has already explained some other data point $\vec z_j$ using differentially private gradient decent algorithm such that $\|\vec z_h - \vec z_j\|<d$ in line~\ref{line-6}. If such $\vec z_j$ exists, it outputs $\phi^\priv(\vec z_j)$ for $\vec z_h$ without any privacy spending, which is sufficiently accurate (Theorem \ref{thm:adpative_closness}). We note that if the explanation of $\vec z_j$ was not computed from scratch (used explanation for some other $\vec z_{j'}; j'<j$) then we can not use the explanation of $\vec z_j$ for $\vec z_h$ because $\|\vec z_{j'}-\vec z_h\|$ might be $>d$. This is taken care by $\flag() $ (see second condition of ``if" statement in line~\ref{line-5} of Algorithm~\ref{algo:APIEP}).
 
If there is no such $\vec z_j$, then our adaptive algorithm selects parameters for \DPGrad{} adaptively exploiting past explained queries such that it converges faster with less privacy spending and lower approximation loss using Theorem~\ref{thm:small_grad_priv}. The adaptive parameter selection using history is explained in Algorithm~\ref{algo:APIEP_para}. Given current explanation query, it picks $\phi^\priv(\vec z_j)$ with minimum $\|\nabla \cal L (\phi^\priv(\vec z_j),\vec z_h)\|$ as an initialization point for the gradient descent algorithm using a differentially private exponential mechanism \cite{dwork2006differential}. It then computes the number of iterations required depending on the selected initialization point according to Theorem~\ref{thm:small_grad_priv}.
 
Algorithm~\ref{algo:APIEP_para} spends $\epsilon_{ite}$ differential privacy budget to choose the starting point.  However, even if ${\|\nabla \cal L (\phi^\priv(\vec z_j),\vec z_h)\|\leq 1/\log T}$ we need to run only $\sqrt \beta$ fraction of total iterations of the differentially private gradient descent algorithm and saves at least a $\beta^\frac{1}{4}$ factor of the privacy budget spent with lower approximation loss (Theorem~\ref{thm:small_grad_priv}).    

Algorithm~\ref{algo:APIEP} is $(\epsilon,\delta)-$ differentially private: all computation steps are conducted in a differentially private manner, and the total privacy budget spend does not exceed $\epsilon$ and $\delta$. 
 
Theorem~\ref{thm:uti_bound_adaptive} shows that Algorithm \ref{algo:APIEP} achieves a nearly optimal approximation when $n\ll m$. Algorithm~\ref{algo:APIEP} is more efficient when the size of the explanation dataset $m$ is much larger than the number of features $n$; this decreases  $\sqrt{n} \sigma_{\textit{min}}$, allowing us to exploit smaller values of $\beta$. 
This trade-off between privacy saving and dimensionality is intuitive: in order to achieve a desirable level of privacy, differentially private gradient descent adds noise proportional to $\frac{\sqrt{n}}{m}$ at each iteration. 
  
We can observe that the gain in approximation loss becomes negligible as $t$ increases in comparison to the required extra privacy budget for that gain. 
More formally, the privacy budget for $t$ iterations is $\cal O(\sqrt t\epsilon_{ite})$ by the Composition Theorem. This implies that the required privacy budget from iteration $t$ to $t+1 $ is $\cal O(\epsilon_{ite}/\sqrt{t})$; on the other hand, the gain in approximation loss from iteration $t$ to $t+1$ is $\cal O (1/t^{\frac{3}{2}})$. 
Therefore we can stop after a small number of iterations. For example, for a dataset of size $10^8$, the optimal algorithm (in terms of approximation loss) suggests running $10^4$ iterations, offering an approximation loss bounded by $0.19$, whereas for $1000$ iterations, the approximation loss is bounded by $0.22$. The difference in privacy spending between the two instances is roughly $68\epsilon_{ite}$. Therefore, for practical purposes, one can use fewer iterations to save more privacy budget with only a negligent deprecation in approximation loss. 

\input{alg-adaptive}
 
 \begin{restatable}{theorem}{thmutiboundadaptive}{\normalfont (Utility Loss)}\label{thm:uti_bound_adaptive}
 Let $\phi^\priv(z_1),\dots \phi^\priv(z_h)$ be the output of the Algorithm~\ref{algo:APIEP}, then for all $j = 1,\dots,h$; $\cal E (\phi^\priv(z_j),\vec z_j) \in \cal O \left(\frac{ \log T }{\sqrt T}  \right)$.  Moreover, for lower dimensional setting ($m >\frac{n}{\epsilon_{ite}} $) and ${T = \sqrt m}$: $\cal E (\phi^\priv(z_j),\vec z_j) \in \cal O \left( \log m/ m^{\frac{1}{4}}  \right)$
 \end{restatable}
 
\begin{proof}
 See Appendix~\ref{appendix:proof_of_theorem_uti_bound_adaptive}
 \end{proof}

\subsection{Early Termination and an Enhanced Adaptive Algorithm}\label{sec:APIEP-early-termination}

Algorithm~\ref{algo:APIEP} is more efficient when the size of the explanation dataset $m$ is much larger than the number of features $n$; this decreases  $\sqrt{n} \sigma_{\textit{min}}$, allowing us to exploit smaller values of $\beta$. 
This trade-off between privacy saving and dimensionality is intuitive: in order to achieve a desirable level of privacy, differentially private gradient descent adds noise proportional to $\frac{\sqrt{n}}{m}$ at each iteration. 
Therefore, in high-dimensional settings, the noise required for privacy dominates the lower norm of the initial gradient and prevents Algorithm \ref{algo:APIEP} from saving more of the privacy budget. 

During the gradient descent optimization phase, if $\|\nabla \cal L(\phi^{\{t\}}) \| \leq \sqrt n\sigma_{\min}$, then Gaussian noise starts dominating $\nabla \cal L(\phi^{\{t\}})$. 
This domination by random Gaussian noise prevents further improvement in the loss function. 
Thus, the extra privacy budget spent once ${\|\nabla \cal L(\phi^{\{t\}}) \| \leq \sqrt n\sigma_{\textit{min}}}$ does not significantly improve explanation quality (approximation error); rather, it starts oscillating around the optimal point, resulting in slower convergence/decrease in approximation loss. 
Therefore, it is not beneficial to spend additional privacy budget once $\|\cal L(\phi^{\{t\}})\|<\sqrt n \sigma$; this observation is confirmed in Figure~\ref{fig:convergence_of_apiep}. 
This motivates us to define an unrestricted version of Algorithm \ref{algo:APIEP}, which maximizes possible savings by the initial point by iterating for $\max \{\beta^{1-\frac{1}{2a}}T,2\}$ times for all queries, where $a$ is the solution to $\beta = \frac{1}{\log^a T}$. The main intuition behind this approach is that whenever adaptive algorithm finds an initialization with $\beta = \| \cal L(\phi^{\{0\}})\| \ll \sqrt{ n} \sigma_{\textit{min}}$ then increasing the number of iterations does not result in faster convergence as Gaussian noise dominates: privacy spending in these cases offers little improvement in loss.  

To implement the enhanced adaptive algorithm, we only change Line~\ref{line_13} to $T'\gets \beta^{1-1/2a}T$ (instead of $ \sqrt n\sigma^{1 - 1/2a}T$) in Algorithm~\ref{algo:APIEP_para} where $a=\frac{\log 1/\beta}{\log \log T}$.  
We do not maintain a similar general theoretical bound on the approximation error as in Algorithm \ref{algo:APIEP};
however, we empirically analyze the performance of the enhanced adaptive algorithm in Section~\ref{sec:experiments} (see Figures~\ref{fig:pipe_vs_apiep-ACS13} and \ref{fig:pipe_vs_apiep-Text} for privacy spending and approximation loss).

\subsection{Last Phase: Non-Interactive DP Mechanisms for Model Explanation}\label{sec:non-interactive}

While Algorithm \ref{algo:APIEP} spends less privacy budget than Algorithm \ref{Algorithm-1}, it will eventually exhaust its privacy budget after explaining finitely many datapoints. 
However, by explaining a sufficiently large number of datapoints, we gather enough information to generate explanations for new queries, {\em without spending any additional privacy budget}. 

Let $\cal H$ be the history of the explanation queries and their explanation generated by Algorithm \ref{algo:APIEP} using a total privacy budget of $(\epsilon,\delta)$. We propose a {\em non-interactive phase} algorithm for generating model explanations, which takes  $\cal H$ the output of Algorithm \ref{algo:APIEP} as input, and generates an explanation for new queries without spending any {\em additional} privacy budget. 
The main idea of this approach is that if $\cal H$ contains enough information about the underlying black-box model that, making additional queries to the explanation dataset $\cal X$ unnecessary. 
We can rather use the explanations already given to the user (adversary) and their own dataset to explain additional queries.

\input{alg-noninteractive}

We construct a {\em proxy explanation dataset} using the history $\cal H$ which contains explained datapoints and their corresponding differentially private explanation. The proxy explanation dataset is simply $\cal X' = \{\vec z_1,\dots \vec z_h: \vec z_j \in \cal H \}$, i.e. the points queried by the user. Their labels are the corresponding model explanations, which are linear approximations of the original model: $\hat{f} (\cal X') = \{ \phi^{\priv}( \vec z_j)^\top \cdot \vec z_j : j= 1,\dots ,h \}$. 
Given a new query $\vec z$, we generate a new query via a linear approximation of the black box model around $\vec z$ using $\cal X'$ and the corresponding differentially private approximation $\hat{f} (\cal X')$. 
\begin{equation}
    \phi^\priv(\vec z,\cal H) = \argmin_{\phi \in \cal C} \sum_{\vec z_j \in \cal X'} \alpha(\|\vec z_j - \vec z \|) (\phi^\top \cdot (\vec z_j - \vec z) - \hat{f}(\vec z_j))^2 \label{eqn:extended_adaptive}
\end{equation}

The non-interactive phase is described in Algorithm~\ref{algo:APIEP_non_interactive} with details. It takes $\cal H$ history as an input which is consists of differentially private explanations of past queries and computes explanations for further explanation queries. The non-interactive phase generates an explanation for a new query $\vec z$ by optimizing~\eqref{eqn:extended_adaptive} which provides a linear approximation of the black-box model around the point $\vec z$. However, we do not need to spend any privacy budget for generating a private explanation for $\vec z$ as it does not use the explanation dataset $\cal X$ for any further computations and $\hat{f}(\cdot)$ is already computed using differential privacy \cite{dwork2006differential}. This allows the adaptive version to explain infinitely many queries even after the entire privacy budget has been spent. We evaluate the performance of the non-interactive phase in Section~\ref{sec:experiments}.

%% file: alg-parameter-selection.tex
\begin{algorithm}[t!]
\small
\SetAlgoLined
\let\oldnl\nl
\newcommand{\nonl}{\renewcommand{\nl}{\let\nl\oldnl}}
\SetKwProg{proc}{Procedure}{}{}
\LinesNumbered
\SetNlSty{}{}{:}

\nonl \textbf{Input: } $\vec z_h \in \R^n$, explanation dataset $\mathcal{X}$, History $\cal H$, privacy spending for parameter selection $\epsilon_{\textit{para}}$, and the number of GD steps $T$, minimum variance $\sigma_{\textit{min}}$\;

\nonl \textbf{Output: } Input variables for \DPGrad{} for the query $\vec z_h$\;

\proc{\Parameters{$\vec z_h$, $ \cal H$, $\epsilon_{\textit{para}}, \sigma_{\textit{min}},\mathcal{X}, T$}}{

\uIf{$\cal H == \emptyset$}{ \label{line-5}
Arbitrary $\phi\in \cal C_{2,1}$\;
\KwRet{$\phi, \sigma_{\textit{min}}, T$}\;
}

\Else{

$\phi^\best \gets \phi^{\priv}(\vec z_j)$ with  $\Pr \propto \textit{exp}\Big(-m\cdot\epsilon_{\textit{para}}\cdot\frac{\|\nabla\cal L(\phi^\priv(\vec z_j),\vec z_h)\|}{2}\Big)$ for $\vec z_j \in \cal H$

$\beta \gets \|\nabla\cal L(\phi^\best,\vec z_h)\|$\;

$\sigma \gets  \max \left(  \frac{\beta}{\sqrt n},\sigma_{\textit{min}} \right)$\;\label{line_13}

$a \gets \frac{\log\frac{1}{\sqrt n\sigma}}{\log\log T}$\tcp*{Thm. \ref{thm:small_grad_priv}}
\uIf{$a>\frac{1}{2}$}{
 
$T'\gets (\sqrt{n}\sigma)^{1 - \frac{1}{2a}} T$\; 
}

\Else{
$T'\gets T$\;
}

}
\KwRet $\phi^{\textit{best}},\sigma, T'$\;
}

\caption{\small Adaptive DP for Parameter Selection}\label{algo:APIEP_para}
\end{algorithm}

%% file: alg-adaptive.tex
\begin{algorithm}[t!]
\small
\SetAlgoLined
\let\oldnl\nl
\newcommand{\nonl}{\renewcommand{\nl}{\let\nl\oldnl}}
\LinesNumbered
\SetKwProg{proc}{Function}{}{}
\SetNlSty{}{}{:}

\nonl \textbf{Input: } Queries $\{ \vec z_1,\dots\} \in \R^n$ arriving one by one, explanation dataset $\mathcal{X}$, privacy budget $(\epsilon, \delta)$, the minimum per-query privacy loss $(\epsilon_{\textit{min}},\delta_{\textit{min}})$, and the number of GD steps $T$\;

$\cal H \gets \emptyset$, $\epsilon_{\spent},\delta_{\spent} \gets 0$\; 

$\epsilon_{ite}\gets \frac{\epsilon_{\textit{min}}}{\sqrt{8T\log\frac{2}{\delta_{\textit{min}}}}}$\tcp*{Privacy budget to spend per iteration} 

$\sigma_{\textit{min}} = \frac{\sqrt{2\log(2.5T/\delta_{\textit{min}})}}{m \cdot \epsilon_{ite}} $ \tcp*{Variance needed for Gaussian mechanism}

$d \gets \frac{\log T}{\sqrt T}$ \tcp*{Distance bound required according to Thm. \ref{thm:adpative_closness}}

\For{$h = 1,\dots,\infty$}{

\uIf{$\exists \vec z_j \in \cal H $ with $\|\vec z_h -\vec z_j\| \leq d$ \& \flag{$\vec z_j$}$ = \top$}{ \label{line-6}
$\phi^\priv(\vec z_h) \gets \phi^\priv(\vec z_j)$\tcp*{Use a nearby point (Thm. \ref{thm:adpative_closness})}
\textbf{report: }$\phi^\priv(\vec z_h)$\tcp*{Report explanation of $\vec z_h$}
$\cal H\dotappend(\vec z_h:\phi^\priv(\vec z_h),$\flag{$\vec z_h$}$ = \bot)$
}

\Else{

$\phi^{\best},\sigma,T' \gets$ \Parameters{$\vec z_h$, $ \cal H$, $\epsilon_{\textit{ite}}, \sigma_{\textit{min}},\mathcal{X}, T$}\;
$\phi^{\priv}(\vec z_h) \gets$ \DPGrad{$\phi^{\best},\sigma,T'$}\;

Update $\epsilon_{\spent},\delta_{\spent}$\tcp*{via the Strong Composition Theorem}
\If{$\epsilon_{\spent}>\epsilon$ or $\delta_{\spent}\geq \delta$ }{
 \textbf{break}\tcp*{Privacy budget is exhausted}
}
\textbf{report: }$\phi^\priv(\vec z_h)$\;
$\cal H\dotappend(\vec z_h:\phi^\priv(\vec z_h),$\flag{$\vec z_h$}$ = \top)$\;

}

}

\caption{\small Adaptive DP for Model Explanation}\label{algo:APIEP}
\end{algorithm}

%% file: alg-noninteractive.tex
\begin{algorithm}[t!]
\small
\SetAlgoLined
\let\oldnl\nl
\newcommand{\nonl}{\renewcommand{\nl}{\let\nl\oldnl}}
\LinesNumbered
\SetKwProg{proc}{Function}{}{}
\SetNlSty{}{}{:}

\nonl \textbf{Input: } Privately explained history $\cal H$ where $h = |\cal H|$, List of point of interest queries $\{ \vec z_h,\dots\} \in \R^n$ arriving one by one\;

$\cal X'=\{\vec z_1,\dots \vec z_{h} \},\hat f(\cdot)$ \tcp*{According to section~\ref{sec:non-interactive}}
\For{i = h+1,\dots }{
$\phi^\priv(\vec z_i) =  \argmin_{\phi \in \cal C_{2,1}} \cal L(\phi, \vec z_h , \cal X',\hat f(\cdot))$\;
\textbf{report: }$\phi^\priv(\vec z_i)$\tcp*{Report explanation of $\vec z_i$}

}

\caption{\small Non-Interactive DP for Model Explanation}\label{algo:APIEP_non_interactive}
\end{algorithm}

%% file: dptraining.tex
\section{Protecting the Privacy of the Training Dataset}\label{sec:privacy-training-dataset}

The analysis in Sections~\ref{sec:PIEP} and \ref{sec:APIEP} shows that our mechanisms protect the \coloremph{explanation dataset}. In what follows, we analyze their \coloremph{training data} ($\cal T$) privacy guarantees.
 
The only potential source of training data information leakage is the gradients computed by our algorithm at each iteration. 
Our algorithms inject Gaussian noise to the gradient of the loss function at each iteration; fortunately, this noise injection offers some privacy guarantees with respect to the training dataset. We quantify the improvement in the privacy guarantees of the training dataset due to the randomness inserted by our explanation algorithms when the underlying training process is $(\trainingepsilon,\trainingdelta)$-differentially private. By the post-processing property of the differential privacy \cite{dwork2006differential}, explanations generated by our algorithm are already $(\trainingepsilon,\trainingdelta)$-differentialy private, however, in Theorem~\ref{thm:dptrain}, we show that our explanations improve the privacy guarantees for the training dataset. We also show that even if the training process is non-differentialy private (training process doesn't provide any protection for the training dataset), our explanation algorithms provide $(\mathcal O (m\epsilon),\delta)$-differential privacy.

\begin{restatable}[Privacy amplification for the training dataset]{theorem}{thmtrainprivimprove}\label{thm:dptrain}
Suppose that the underlying training process of the black-box model $f$ is $(\trainingepsilon,\trainingdelta)$-differentially private w.r.t. the training dataset, then the explanation algorithm described in Algorithm~\ref{Algorithm-1} with parameters $(\epsilon, \delta)$ is $(\trainingepsilon,\gamma\trainingdelta)$-differentilly private for the training dataset, where, $\gamma \in  1-\mathcal O \left(\left(\frac{\log \frac{T}{\delta}}{m\epsilon }\right)\left (\euler^{-\frac{m\epsilon}{ \log \left(\frac{T}{\delta}\right)}}\right )\right)$. Moreover for any training process, Algorithm~\ref{Algorithm-1} is $(\mathcal O (m\epsilon),\delta)$-differentially private for the training dataset.
\end{restatable}
\begin{proof}
See Appendix~\ref{appendix:proofofdptrain}
\end{proof}

Theorem~\ref{thm:dptrain} (proven in Appendix~\ref{appendix:proofofdptrain}) shows that our privacy guarantees with respect to the training set grow weaker as the size of the explanation dataset (the parameter $m$) increases; this presents a natural tradeoff between the amount of data used by the mechanisms generating model explanations, and the privacy guarantees we can offer with respect to the training data. 
 
At first glance, our guarantees seem rather weak - growing linearly worse in $m$. However, our bound is tight: consider the (extremely unstable) training where the label of all datapoints in the dataspace depends on the presence of a single data record in the training dataset; if this point is present then all labels are $+1$, and are $-1$ otherwise. Protecting against information leaks in this example is nearly impossible, irrespective of the model explanation mechanism \cite{shokri2019privacy, shokri2017membership}. 

%% file: experiments.tex
\section{Experimental Results}\label{sec:experiments}

We evaluate our model explanations on standard machine learning datasets.  We use the following benchmark datasets (we use the datasets in their entirety as the explanation dataset $\cal X$): 
  
\textbf{ACS13:} We use a scrubbed version of the dataset used in \cite{bindschaedler2017plausible}\footnote{Pulled from \url{http://www.census.gov/programs-surveys/acs/}.} containing $1,494,974$ records and predict income ($>50\text{k} \$ $ vs $\leq 50\text{k} \$ $). We train a random forest classifier with $500$ trees with maximum depth $= 10$, which achieves $85\%$ training accuracy and $84\%$ test accuracy. 
  
\textbf{IMDB/Amazon Movie Reviews (Text dataset) \cite{maas2011learning,amazondata}:}  This dataset consists of $8,765,568$ movie reviews from the Amazon review dataset along with $50,000$ movie reviews from IMDB large review dataset mapped to binary vector using the top $500$ words. Each movie review is labeled as either a positive ($+1$) or a negative ($-1$) review. We consider the true labels as the output of some unknown black-box algorithm. 

\textbf{Facial expression dataset  \cite{goodfellow2015challenges}} This dataset consists of $12,156$ $48 \times 48$ pixel grayscale images of faces. We train a CNN with two convolution layers with $5\times 5$ filters followed by max-pooling and a fully connected layer, achieving training and test accuracy of $86\%$ and $84.3\%$, respectively. We use this dataset to demonstrate the visualization of our explanations.

We repeat experiments $10$ times, and plot their average behavior. 
\subsection{Interactive DP Model Explanation}\label{sec:personalized-explanations}
We use our non-adaptive protocol (Algorithm~\ref{Algorithm-1}) to generate private model explanations for points in all datasets. We compare the explanations we generate with existing non-private model agnostic explanation methods.
In the Text dataset, we generate differentially private model explanations for $1000$ randomly sampled datapoints (movie reviews) from the dataset with strong privacy parameters $\epsilon = 0.1$ and $\delta = 10^{-6}$. We present a few examples in Table~\ref{tab:PIEP_for_IMDB}.  The explanations generated by our protocol agree with well-established non-private model agnostic explanations: using LIME \cite{ribeiro2016should} and MIM \cite{sliwinski2019axiomatic}, we extract the top $5$ most influencial words. According to our evaluation, our protocol and MIM share $2.6$ of the top $5$ influential words on average, whereas our protocol and LIME share $3.9$ words on average, with a variance of less than $0.3$ in both cases.  
\begin{table*}
	
	\caption{\small Examples of influence measures generated for Text movie reviews dataset by Algorithm~\ref{Algorithm-1} with $\epsilon \approx 0.1$ and $\delta =10^{-6}$, LIME and MIM. Upwards (downwards) arrows indicate a high positive (negative) influence of a word. Moreover blue, red and green arrows correspond to words selected to Algorithm~\ref{Algorithm-1}, LIME, and MIM.}
	\centering
	
	\label{tab:PIEP_for_IMDB}
	\begin{tabular}{lp{0.8 \textwidth}c}
		\toprule
		&\textbf{Movie Review}&\textbf{label}   \\
		\midrule
		1.& ... year Batman ... attempted make \bestMIM{\bestPIEP{well}} acted De Vito ... \worstMIM{\worstLIME{\worstPIEP{bad}}} intentions ... searching past \worstPIEP{\worstMIM{would}} like
		... given \worstMIM{\worstLIME{\worstPIEP{bad}}} reviews... &+1 \\
		\midrule
		2. &...\bestMIM{\bestLIME{\bestPIEP{superb}}} performance by Natalie Portman... saying script \worstPIEP{bad} at times but I \worstMIM{\worstPIEP{don't}}...The film \bestPIEP{look} bad, don't  \bestPIEP{good}
		direction and \bestLIME{\bestPIEP{excellent}} \bestMIM{\bestLIME{performances}}...&+1   \\
		\midrule

		3. &Yeah adults may find \bestMIM{\bestLIME{\bestPIEP{stupid}}}... \bestMIM{\bestPIEP{don't}} think really bad.... The \worstPIEP{story} aAlvin gang... across world search jewels {\bestMIM{\bestLIME{\bestPIEP{bad}}}} .... 
		with... So animation {\worstPIEP{good}} ...&-1\\ 
		\midrule
		4. &I never seen such {\bestMIM{\bestPIEP{horrible}}} special affects or
		acting... I \worstMIM{\worstLIME{\worstPIEP{laughed}}} so hard on this its just {\bestPIEP{stupid}} 
		I mean the movie is so  \bestPIEP{\bestLIME{\bestMIM{awful}}}... &-1\\
		
		\bottomrule
	\end{tabular}
\end{table*}

How does explanation quality degrade as we make our privacy requirements more stringent? This can be visually observed for the facial expression dataset by generating explanation by Algorithm~\ref{Algorithm-1} with different $\epsilon$ values, and a fixed $\delta = 10^{-5}$ (Figure~\ref{PIEP:face_for_different_epsilon}); Algorithm~\ref{Algorithm-1} generates explanations that appear meaningful with $\epsilon \geq 0.07$ and $\delta = 10^{-5}$. 

\begin{table}
\centering
  \caption{\small Summary of mean $\pm $ Variance of loss in utility for each dataset for explanation generated by Algorithm~\ref{Algorithm-1} (Theorem~\ref{thm:uti_bound}).}
  \label{tab:loss_summary}
  \begin{tabular}{lcc}
    \toprule
    Dataset& $\epsilon = 0.01,\delta = 10^{-6}$   &  $\epsilon = 0.1,\delta = 10^{-6}$  \\
    \midrule
     Face & $1.4 \times 10^{-2} \pm 4.3 \times 10^{-3}$ & $2.2 \times 10^{-3} \pm 2.1 \times 10^{-4} $  \\
    \midrule
     Text & $2.7 \times 10^{-3} \pm 7.8 \times 10^{-4}$ & $4.3 \times 10^{-4} \pm 9.4 \times 10^{-6} $ \\
    \midrule
     ACS13& $5.7 \times 10^{-3} \pm 1.3 \times 10^{-3} $ & $ 2.6 \times 10^{-4} \pm 6.3 \times 10^{-5} $ \\
  
    \bottomrule
  \end{tabular}
\end{table}

\begin{figure}
	\includegraphics[width=\linewidth]{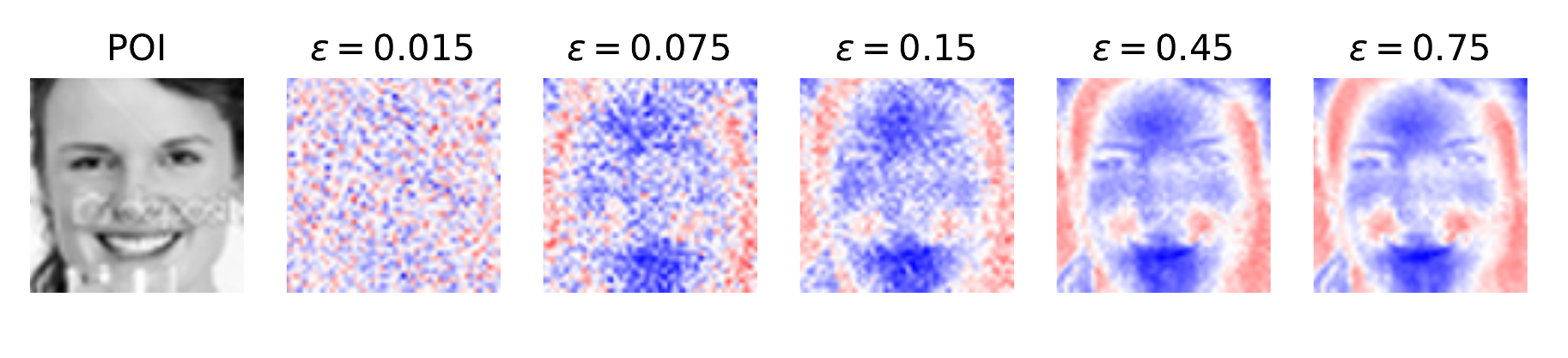}
	\caption{\small The effect of varying $\epsilon$  on explanation quality of Algorithm~\ref{Algorithm-1}. The color blue (red) indicates positive (negative) influence. Brighter colors indicate greater influence.}\label{PIEP:face_for_different_epsilon}
\end{figure}
 
We compare the loss of our explanation algorithm defined in Equation~\eqref{eq:utility} for different privacy parameters $\epsilon$ on $1000$ randomly selected datapoints from the datasets. We compute the mean $\pm$ variance value of the approximation loss for ACS13 and Text dataset for $\epsilon=0.01$ and $\epsilon = 0.1$ with $\delta=10^{-6}$ (Table~\ref{tab:loss_summary}). 
The loss decreases as we relax privacy requirements (this follows from Theorem~\ref{thm:uti_bound}). We compute better explanations for the Text dataset than the ACS13 dataset; this can be explained by their different dimensionality (the $\sqrt n/m$ factor in Theorem~\ref{thm:uti_bound}). 

\begin{figure}[htbp]
	\centering
	\begin{subfigure}[t]{0.45\textwidth}
		\scalebox{0.7}
		{
		\begin{tikzpicture}
			\begin{axis}
			[ ymin=0, ymax=1.1,xlabel={Number of Queries},ylabel={Privacy Budget Spent $\epsilon$}, grid = major]
			\addplot [ mark=o,color=red] table [x index=0, y index =1] {"data_files_for_plotting/ACS_data_C_PIEP.dat"}node[above,xshift=-0.5cm]{\small Non-Adaptive};
			\addplot[ mark=x,color=blue] table [x index=0, y index =1] {"data_files_for_plotting/ACS_data_C_APIEP.dat"}node[above,xshift=-0.5cm]{\small Adaptive};
			\addplot [ mark=+,color=black!50!green] table [x index=0, y index =1] {"data_files_for_plotting/ACS_unrest_APIEP_epsi.dat"}node[above,xshift=-0.7cm]{\small Enhanced-Adaptive};
			\end{axis}
		\end{tikzpicture} 
		}
		\caption{\tiny Total privacy budget spent by the non-adaptive (Algorithm~\ref{Algorithm-1}), adaptive (Algorithm~\ref{algo:APIEP}) and enhanced adaptive algorithms on the same sample for the ACS13 dataset.}\label{fig:piep_vs_apiep-ACS13-privacy}
	\end{subfigure}
	\begin{subfigure}[t]{0.45\textwidth}
		\scalebox{0.7}
		{
		\begin{tikzpicture}
			\begin{axis}
			[legend style={font=\small}, xlabel={Loss function value}, ylabel  ={Normalized Frequency}, legend entries = {Adaptive,Enhanced-Adaptive  , Non-Interactive Phase }, grid = major,
			legend pos = north east]
			x tick label style={fixed,precision=2}
			\addplot[ybar interval,mark=no,blue] table [x index=1, y index =0] {"data_files_for_plotting/ACS_loss_hist_APIEP_2.dat"};
			\addplot[ybar interval,mark=no,black!50!green] table [x index=1, y index =0] {"data_files_for_plotting/ACS_loss_hist_unrestricted_APIEP_2.dat"};
			\addplot[ybar interval,mark=no,orange] table [x index=1, y index =0] {"data_files_for_plotting/ACS_loss_hist_APIEP_free_2.dat"};
			\end{axis}
		\end{tikzpicture}
	}
		\caption{\tiny Normalized loss histogram for queries explained by the adaptive, enhanced adaptive and non-interactive mechanisms for the ACS13 dataset. Adaptive approaches exhibit lower loss, trading off privacy and utility.}\label{fig:hist_loss_allmethod-ACS13}
	\end{subfigure}		
	\caption{\small  Explanation quality and privacy guarentees for the ACS13 dataset.
	}\label{fig:pipe_vs_apiep-ACS13}
	
\end{figure}
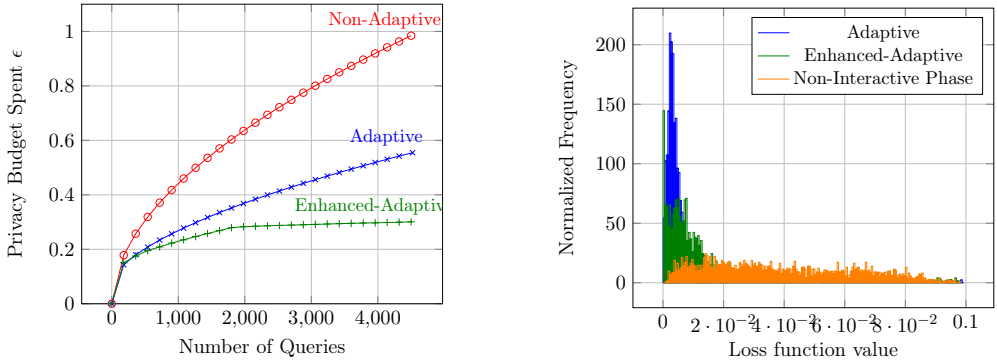

\begin{figure}[htbp]
	\centering
	\begin{subfigure}[t]{0.42\textwidth}
		\scalebox{0.7}
		{
		\begin{tikzpicture}
			\begin{axis}
			[ xlabel={Number of Queries}, ylabel  ={Privacy Budget Spent $\epsilon$}, grid = major]
			\addplot[ mark=o,color=red] table [x index=0, y index =1]{"data_files_for_plotting/Amazon_data_C_PIEP.dat"}node[above,xshift=-0.7cm]{\small Non-Adaptive};
			\addplot[ mark=x,color=blue] table [x index=0, y index =1] {"data_files_for_plotting/Amazon_data_C_APIEP.dat"}node[above,xshift=-0.7cm]{\small Adaptive};
			
			\addplot[ mark=+,color=black!50!green] table [x index=0, y index =1]{"data_files_for_plotting/Amazon_data_C_U_APIEP.dat"}node[above,xshift=-0.7cm]{\small Enhanced-Adaptive};
			
			\end{axis}
		\end{tikzpicture} 
		}
		\caption{Total privacy budget spent by the non-adaptive (Algorithm~\ref{Algorithm-1}), adaptive (Algorithm~\ref{algo:APIEP}) and enhanced adaptive algorithms on the same sample for the Text dataset.}\label{fig:pipe_vs_apiep-Text-privacy}
		\end{subfigure}
		\begin{subfigure}[t]{0.42\textwidth}
		\scalebox{0.7}
		{
			\begin{tikzpicture}
				\begin{axis}
				[ legend style={font=\small},xlabel={Loss function value}, ylabel  ={Normalized Frequency}, grid = major, legend entries = {Adaptive,Enhanced-Adaptive  , Non-Interactive Phase},  legend pos =  north east]
				x tick label style={fixed,precision=2}
				\addplot[ybar interval,mark=no,blue] table [x index=1, y index =0] {"data_files_for_plotting/Text_loss_hist_APIEP_2.dat"};
				\addplot[ybar interval,mark=no,black!50!green] table [x index=1, y index =0] {"data_files_for_plotting/Text_loss_hist_unrestricted_APIEP_2.dat"};
				\addplot[ybar interval,mark=no,orange] table [x index=1, y index =0] {"data_files_for_plotting/Text_loss_hist_APIEP_free_2.dat"};
				\end{axis}
			\end{tikzpicture}
		}
			\caption{Normalized loss histogram for queries explained by the adaptive, enhanced adaptive and non-interactive mechanisms for the Text dataset.}\label{fig:hist_loss_allmethod-Text}
		\end{subfigure}
		\caption{Sub-figure \ref{fig:pipe_vs_apiep-Text-privacy} shows the total privacy budget spent by the non-adaptive (Algorithm~\ref{Algorithm-1}), adaptive (Algorithm~\ref{algo:APIEP}) and enhanced adaptive algorithms on the same sample for the Text dataset. The $x$-axis represents the number of queries answered, and the $y$-axis shows the value of privacy parameter $\epsilon$. Sub-figure~\ref{fig:hist_loss_allmethod-Text} is the normalized loss histogram for queries explained by the adaptive, enhanced adaptive and Non-interactive phase. The adaptive approach has a lower loss than the non-interactive phase, a natural trade-off between privacy and utility.}\label{fig:pipe_vs_apiep-Text}
\end{figure}
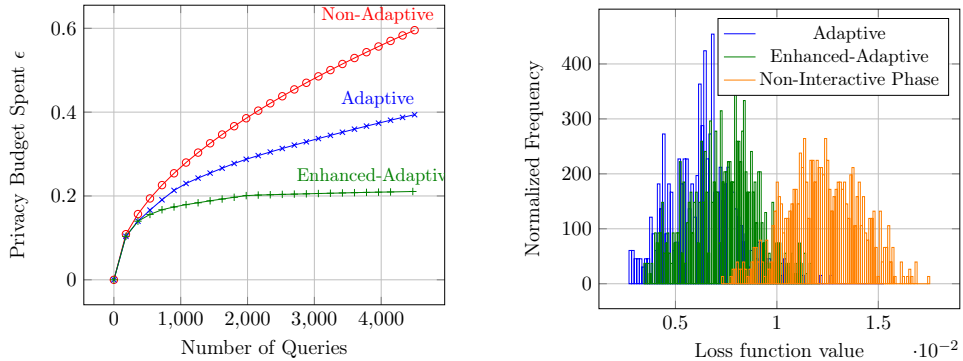

\subsection{Saving the Privacy Budget via the Adaptive Approach} \label{sec:apiep-savings}

We analyze both privacy spending and loss in our adaptive protocols, as compared to the baseline non-adaptive protocol (Algorithm~\ref{Algorithm-1}). 
We randomly sample three batches of $1500$ datapoints and divide the explanation dataset into three equal-sized, disjoint parts. 

We run parallel composition \cite{mcsherry2009privacy} for our private explanations by separately using the three disjoint explanation datasets to explain three sequences of $1500$ queries. We predefine a desired level of loss guarantee; in other words, both model explanations must achieve the same explanation quality, to ensure a valid comparison. We set the per-query $\epsilon$ parameter to $\epsilon = 0.01$ for the ACS13 dataset, and to $\epsilon = 0.006$ for the Text dataset. In both evaluations we set $\delta = 10^{-7}$, and set the maximal number of \DPGrad{} iterations to $T = 300$. This upper bound is only reached for $94/4500$ ($73/4500$) queries for the Text(ACS13) dataset. 
In all other cases, the adaptive algorithms were able to find a better initialization point. See Figure~\ref{fig:his_no_ite} for more details. 

Interestingly, for the Text (ACS13) data, on average, over $1850$ ($2220$) instances had a better initialization point, which failed to satisfy the $\sqrt{n}\sigma_{\textit{min}}\leq \beta$ condition, and thus had to proceed with $T' = 91 $ ($T'= 95$) (the minimum iterations theoretically required according to Theorem \ref{thm:small_grad_priv}. 
We note that this minimum $T'$ improves as we relax our privacy requirements (Theorem~\ref{thm:small_grad_priv}), allowing the adaptive algorithm to save more of the privacy budget. 
The improved performance for the ACS13 dataset may be explained by the fact that it consists of dense regions, as compared to the Text dataset.

We analyze privacy spending and loss values (Equation~\ref{eq:MSE}) by different algorithms in Figure~\ref{fig:pipe_vs_apiep-ACS13} (ACS13) and Figure~\ref{fig:pipe_vs_apiep-Text} (Text).  
The adaptive protocols spend most of their privacy budget on the initial $\sim 100$ queries; after they generate enough information, they capitalize on it to explain the remaining queries using a smaller per-query privacy budget.  
Furthermore, the adaptive protocols achieve the same explanation quality as the non-adaptive protocol, while utilizing a much smaller privacy budget (as predicted by Theorem~\ref{thm:uti_bound_adaptive}). Figures~\ref{fig:hist_loss_allmethod-ACS13} and \ref{fig:hist_loss_allmethod-Text} show that the Adaptive and Enhanced-Adaptive protocols exhibit a similar distribution of loss values, with a significant difference in the privacy spending rate. This can be explained by the dominance of Gaussian noise over the ``good'' initialization point.   

In Figure~\ref{fig:convergence_of_apiep}, we plot the convergence of \DPGrad for 15 randomly selected data points for which Algorithm~\ref{algo:APIEP} found an initialization point with $\beta < \sqrt{n}\sigma_{\textit{min}}$(line~\ref{line_13} of Algorithm~\ref{algo:APIEP_para}).  
We observe that after the first few iterations, Gaussian noise dominates and increases the loss value; only later does it converge to a loss close to the initial point in such cases. This results in extra spending of privacy without any gain in approximation loss. This empirically explains the similar distribution of explanation quality by adaptive and enhanced adaptive algorithms (Figures~\ref{fig:hist_loss_allmethod-ACS13} and \ref{fig:hist_loss_allmethod-Text}) despite that enhanced adaptive algorithm iterates less number of iterations compare to adaptive algorithm and spends significant lesser privacy budget (Figures~\ref{fig:hist_loss_allmethod-ACS13} and \ref{fig:hist_loss_allmethod-Text}). 

\begin{figure}[ht!]
	\centering
	\resizebox{0.5
	\textwidth}{!}{
		\begin{tikzpicture}
		\begin{axis}
		[ ymin=0,legend style={font=\small},title=  {Loss value over iterations for Algorithm~\ref{algo:APIEP}} ,xlabel={Gradient Descent Iterations},ylabel={Loss Function Value}, grid = major]
		\addplot[ mark=+ ] table [x index=0, y index =1] {"data_files_for_plotting/APIEP_Convergence/ACS__APIEP_private_convergence_50.dat"}; 
		\addplot[ mark=+ ] table [x index=0, y index =1] {"data_files_for_plotting/APIEP_Convergence/ACS__APIEP_private_convergence_51.dat"};
		\addplot[ mark=+ ] table [x index=0, y index =1] {"data_files_for_plotting/APIEP_Convergence/ACS__APIEP_private_convergence_52.dat"};
		\addplot[ mark=+ ] table [x index=0, y index =1] {"data_files_for_plotting/APIEP_Convergence/ACS__APIEP_private_convergence_53.dat"};
		\addplot[ mark=+ ] table [x index=0, y index =1] {"data_files_for_plotting/APIEP_Convergence/ACS__APIEP_private_convergence_54.dat"};
		\addplot[ mark=+ ] table [x index=0, y index =1] {"data_files_for_plotting/APIEP_Convergence/ACS__APIEP_private_convergence_55.dat"};
		\addplot[ mark=+ ] table [x index=0, y index =1] {"data_files_for_plotting/APIEP_Convergence/ACS__APIEP_private_convergence_56.dat"};
		\addplot[ mark=+] table [x index=0, y index =1] {"data_files_for_plotting/APIEP_Convergence/ACS__APIEP_private_convergence_57.dat"};
		\addplot[ mark=+] table [x index=0, y index =1] {"data_files_for_plotting/APIEP_Convergence/ACS__APIEP_private_convergence_58.dat"};
		\addplot[ mark=+] table [x index=0, y index =1] {"data_files_for_plotting/APIEP_Convergence/ACS__APIEP_private_convergence_59.dat"};
		\addplot[ mark=+] table [x index=0, y index =1] {"data_files_for_plotting/APIEP_Convergence/ACS__APIEP_private_convergence_60.dat"};
		
		\addplot[ mark=+] table [x index=0, y index =1] {"data_files_for_plotting/APIEP_Convergence/ACS__APIEP_private_convergence_61.dat"};
		\addplot[ mark=+] table [x index=0, y index =1] {"data_files_for_plotting/APIEP_Convergence/ACS__APIEP_private_convergence_62.dat"};
		\addplot[ mark=+] table [x index=0, y index =1] {"data_files_for_plotting/APIEP_Convergence/ACS__APIEP_private_convergence_63.dat"};
		\addplot[ mark=+] table [x index=0, y index =1] {"data_files_for_plotting/APIEP_Convergence/ACS__APIEP_private_convergence_64.dat"};
		\addplot[ mark=+] table [x index=0, y index =1] {"data_files_for_plotting/APIEP_Convergence/ACS__APIEP_private_convergence_65.dat"};
		
		\end{axis}
		\end{tikzpicture}
		}
	\caption{\small The convergence pattern of DP gradient descent for 15 randomly selected from explained queries for ACS13 dataset by Algorithm~\ref{algo:APIEP} , where $\sigma_{\textit{min}} > \frac{\beta }{\sqrt{n}}$ for the initial point, line~\ref{line_13} of Algorithm~\ref{algo:APIEP_para}. This is where the initial point is close to the optimal point. The noise variance pushes the \coloremph{good} initialization away from the optimal point, only to re-converge to nearly the same level.}  \label{fig:convergence_of_apiep}
\end{figure}
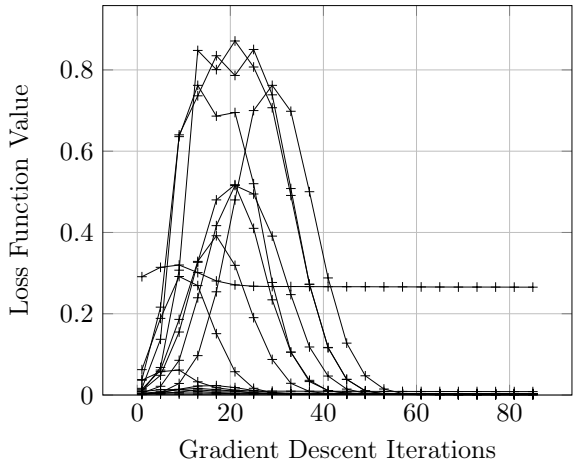

Once Algorithm~\ref{algo:APIEP} explains $4500$ queries, we use its output to generate additional explanations in the Non-Interactive Phase (without spending any further privacy budget) for another randomly sampled $4500$ datapoints for both datasets. The non-interactive phase exhibits higher loss, but spends none of the privacy budget (Figure \ref{fig:pipe_vs_apiep-Text}).

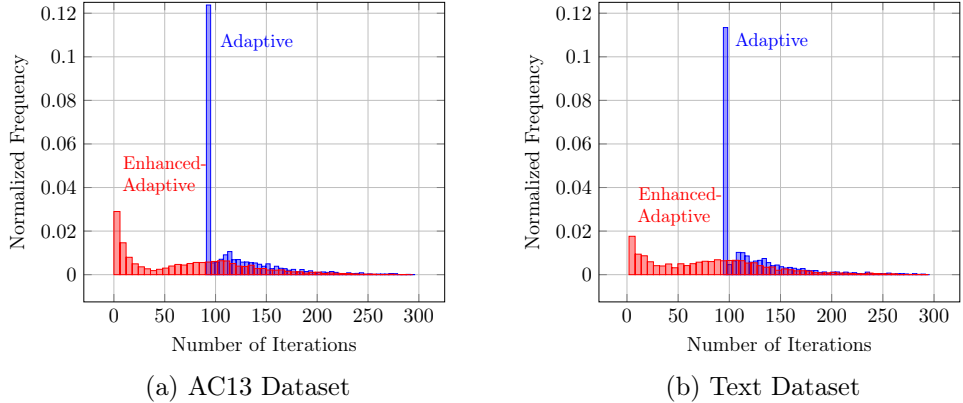
\begin{figure}
	\centering
	\begin{subfigure}[t]{0.42\textwidth}
		\scalebox{0.7}
		{
		\begin{tikzpicture}
			\begin{axis}
			[xlabel={Number of Iterations}, ylabel  ={Normalized Frequency}, yticklabel style={ /pgf/number format/fixed, /pgf/number format/precision=2 }, scaled y ticks=false,ymax=0.125,grid = major ]
			x tick label style={fixed,precision=2}
			\addplot[ybar interval,mark=no,blue,fill,fill opacity=0.4] table [x index=1, y index =0] {"data_files_for_plotting/ACS_hist_ite.dat"} node[above,xshift=-3cm,yshift=4cm,opacity=1]{\small Adaptive};
			\addplot[ybar interval,mark=no,red,fill,fill opacity=0.4] table [x index=1, y index =0] {"data_files_for_plotting/ACS_unrest_hist_ite.dat"} node[above,xshift=-4.8cm,yshift=1.3cm,opacity=1]{\parbox{1.4cm}{\small Enhanced-Adaptive}};
			\end{axis}
		\end{tikzpicture}
		}
		\caption{AC13 Dataset}\label{fig:his_no_ite-ACS13} 
	\end{subfigure}
	\begin{subfigure}[t]{0.42\textwidth}
		\scalebox{0.7}
		{
		\begin{tikzpicture}
			\begin{axis}
			[xlabel={Number of Iterations}, ylabel  ={Normalized Frequency}, yticklabel style={ /pgf/number format/fixed, /pgf/number format/precision=2 }, scaled y ticks=false,ymax=0.125,grid = major ]
			\addplot[ybar interval,mark=no,blue,fill,fill opacity=0.4] table [x index=1, y index =0] {"data_files_for_plotting/Amazon_hist_ite.dat"} node[above,xshift=-3cm,yshift=4cm,opacity=1]{\small Adaptive};
			\addplot[ybar interval,mark=no,red,fill,fill opacity=0.4] table [x index=1, y index =0] {"data_files_for_plotting/Amazon_hist_unrest_ite.dat"}node[above,xshift=-4.8cm,yshift=0.7cm,opacity=1]{\parbox{1.4cm}{\small Enhanced-Adaptive}};
			\end{axis}
		\end{tikzpicture}
		}
		\caption{Text Dataset}\label{fig:his_no_ite-Text}
	\end{subfigure}
	\caption{\small The normalized histogram of the number of DP gradient descent iterations by the Adaptive Algorithm (Algorithm~\ref{algo:APIEP}) and the Enhanced-Adaptive algorithm for the ACS13 and Text datasets.This spike for the Adaptive algorithm is due to minimum required iterations. For both the ACS13 and Text dataset, the Enhanced-Adaptive approach iterates less than $90$ times per query in most instances. 
	The notable spike at around $100$ iterations is due to several samples running the theoretically required number of iterations for Algorithm \ref{algo:APIEP} to provably maintain its quality guarantees.  
	See Theorem~\ref{thm:small_grad_priv} and Line~\ref{line_13} of Algorithm~\ref{algo:APIEP_para} for the related theoretical analysis. 
	Overall, Algorithm~\ref{algo:APIEP} ran $149$ iterations on average for the Text dataset, and $121$ for the ACS13 dataset. The Enhanced-Adaptive algorithm runs for $28$ iterations and $43$ iterations on average, respectively.
	}
	\label{fig:his_no_ite}
\end{figure}

\subsection{Differentially Private Explanation for Differentially Private Models}

We train differentially private neural networks (all with the same architecture) on the IMDB Dataset \cite{maas2011learning} using the Tensorflow differential privacy library \cite{dpcode}; we obtain differentially private models of the form $f_{\trainingepsilon}$ with different privacy parameters $\trainingepsilon = 0.5, 2, 10$ and $\trainingdelta = 10^{-5}$. The models are trained on the same training and test dataset, split in a $70:30$ ratio ($35,000$ training reviews and $15,000$ test reviews). The models $f_{0.5},f_2,f_{10}$ achieve a test accuracy of $52\%, 74\%, 81\%$ respectively. We also train a non-private model, $f$, which achieves a test accuracy of $83.6\%$.

We also train differentially private neural networks (all with the same architecture) on the Face Dataset, with privacy parameters $\trainingepsilon = 0.5, 2, 10$ and $\trainingdelta = 10^{-5}$. They achieve a test accuracy of $51\%, 62\%$ and $80\%$ respectively. We also train a non-private model, $f$, which achieves test accuracy of $84.3\%$. The models are trained on the same training and test dataset split into $70:30$ ratio (number of training images $=8510$ and test images $=3646$).

We generate private explanations (Algorithm~\ref{Algorithm-1} with $\epsilon=0.1,\delta = 10^{-5}$ and $T=200$) for all models for the same randomly sampled $500$ datapoints. We compare the private explanation for the POI for the different private models with the base optimal explanation for the non-private model. We compute the distortion ($\dist$) for each POI and DP models: $\dist(\vec x,f_{\trainingepsilon})= \|\phi^\priv(\vec x, f_{\trainingepsilon}) - \phi^*(\vec x,f)\| $. $\dist$ captures the change in the models' behavior around POI due to noise added in the training process. 
We plot the histogram of the $\dist$ values in Figure~\ref{fig:dpmodelsdpexplanations}(a) observe that the $\dist$ values are much larger for the model $f_{0.5}$, compared to $f_{10}$. This is intuitive: for smaller $\trainingepsilon$, model becomes noisier and uses random features for the classification. We further analyze this visually.  

In Figure~\ref{fig:exampleofdptraining}, we visually represent the explanation generated for different private models for an image which was misclassified as a happy face by all models. 
When the privacy requirements are extremely high ($\trainingepsilon = 0.5$), feature influence is randomly distributed: this informally shows that the modal randomly classified the image as a happy face; however, as we relax the privacy requirements of the training dataset, the model explanation becomes more constructive, utilizing important features. For example, pixels around lips start having high positive influence.
  
We plot the histogram of the approximation error (Equation~\ref{eq:utility}) in Figure~\ref{fig:dpmodelsdpexplanations}(b) and observe that the approximation error is higher for the model $f_{0.5}$. As $f_{0.5}$ becomes noisy, it assigns labels randomly around POI. This unpredictable behavior of the model around POI requires more iterations to converge to the optimal linear approximation. Indeed, differentially private models with strong privacy guarantees are difficult to approximate by the \DPGrad{} procedure spending more of the privacy budget for the same level of approximation loss.

To evaluate the impact of noisy/random model behavior around the POI, and the explanation approximation error, we plot the approximation error vs misclassification error of the model around a POI in Figure~\ref{fig:approx_errorvsloss}. 
For all differentially private models, there is a positive correlation between explanation approximation and model misclassification error. For all models and POIs, we use the same input paramaters to generate private explanations $(\epsilon = 0.1,\delta = 10^{-5}, T= 500 )$. This empirically confirms that when the model ``misbehaves'', randomly assigning labels around the POI, private explanations exhibit a higher approximation error. 
This further shows that differentially private models with strong privacy guarantees require more privacy budget for the explanation dataset to achieve a similar level of explanation quality. This also explains the higher approximation loss for $f_{0.5}$ compared to $f_2,f_{10}$ observed in Figure~\ref{fig:dpmodelsdpexplanations}(b).      

\begin{figure}
    \includegraphics[width=\linewidth]{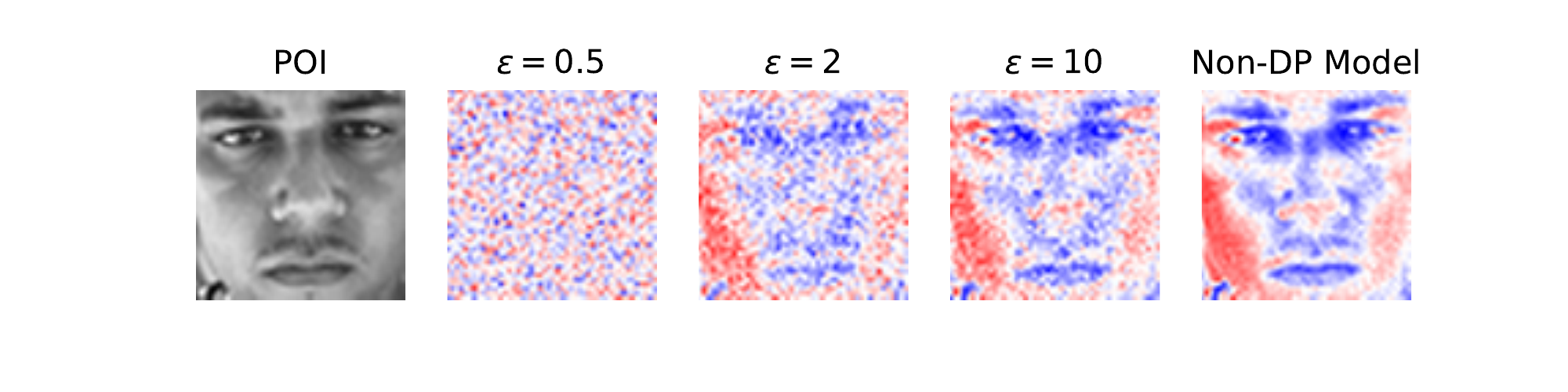}
    \caption{\small The effect of differential private models on the explanation. The color blue (red) indicates positive (negative) influence. Brighter colors indicate greater influence. }\label{fig:exampleofdptraining}
\end{figure}
\begin{figure}
	\centering
	\begin{subfigure}[t]{0.45\textwidth}
		\scalebox{0.7}
		{
		\begin{tikzpicture}
			\begin{axis}
			[xlabel={$\dist$ values}, ylabel  ={Frequency},yticklabel style={ /pgf/number format/fixed, /pgf/number format/precision=2 }, scaled y ticks=false,grid = major ]
			x tick label style={fixed,precision=2}
			\addplot[ybar interval,mark=no,blue] table [x index=1, y index =0] {"data_files_for_plotting/Text_hist_dist_dp_10.dat"} node[above,yshift=4cm,xshift=-0.5cm]{$\epsilon=10$};
			\addplot[ybar interval,mark=no,red] table [x index=1, y index =0] {"data_files_for_plotting/Text_hist_dist_dp_2.dat"}node[above,xshift=0.5cm,yshift=3cm]{$\epsilon=2$};
			\addplot[ybar interval,mark=no,orange] table [x index=1, y index =0] {"data_files_for_plotting/Text_hist_dist_dp_0-5.dat"}node[above,xshift=-1cm,yshift=3cm]{$\epsilon=0.5$};
			\end{axis}
		\end{tikzpicture}
		}
		\caption{\small The $\dist$ values for different models over the same sampled datapoints.}\label{fig:dpmodelsdpexplanations-dist}	
	\end{subfigure}
	\begin{subfigure}[t]{0.45\textwidth}
		\scalebox{0.7}
		{
		\begin{tikzpicture}
			\begin{axis}
			[xlabel={Approximation Error}, ylabel  ={Frequency},yticklabel style={ /pgf/number format/fixed, /pgf/number format/precision=2 }, scaled y ticks=false,grid = major ]
			\addplot[ybar interval,mark=no,blue] table [x index=1, y index =0] {"data_files_for_plotting/Text_hist_approx_error_dp_10.dat"} node[above,xshift=0.5cm,yshift=4cm]{$\epsilon=10$};
			\addplot[ybar interval,mark=no,red] table [x index=1, y index =0] {"data_files_for_plotting/Text_hist_approx_error_dp_2.dat"} node[above,yshift=2.7cm,xshift=0.3cm]{$\epsilon=2$};
			\addplot[ybar interval,mark=no,orange] table [x index=1, y index =0] {"data_files_for_plotting/Text_hist_approx_error_dp_0-5.dat"} node[above,xshift=-1cm,yshift=2cm]{$\epsilon=0.5$};
			\end{axis}
		\end{tikzpicture}
		}
		\caption{\small Explanation approximation loss for different models over the same sampled datapoints.}\label{fig:dpmodelsdpexplanations-loss}
	\end{subfigure}		
	\caption{Histograms of $\dist$ values and explanation quality for different DP models. The approximation error and explanation quality for the model $f_{0.5}$ is higher compared to $f_{2},f_{10}$, likely due to the complicated decision boundaries around POI for noisier models.}\label{fig:dpmodelsdpexplanations}

\end{figure}
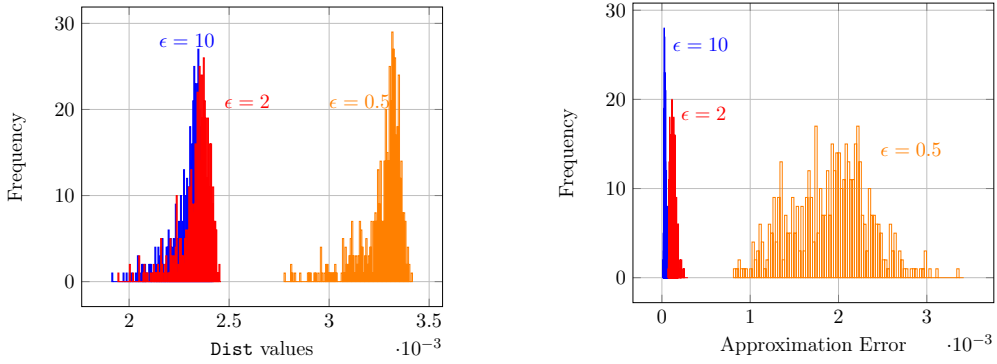

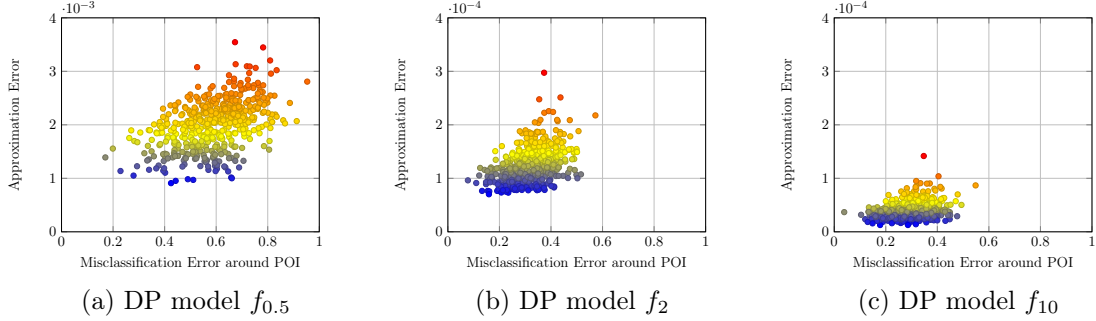
\begin{figure}[ht!]
\centering
	\begin{subfigure}[t]{0.313\textwidth}
		\scalebox{0.5}
		{
		\begin{tikzpicture}
			\begin{axis}
			[ymin=0,ymax=0.004,xmin=0,xmax=1, xlabel={Misclassification Error around POI},ylabel={Approximation Error}, grid = major]
			\addplot [scatter,only marks,color=blue] table [x index=0, y index =1] {"data_files_for_plotting/Text_mis_classification_error_around_POI_vs_explloss_0-5.dat"};
			\end{axis}
		\end{tikzpicture} 
		}
	\caption{DP model $f_{0.5}$}\label{fig:approx-errorloss-0.5}
	\end{subfigure}
	\begin{subfigure}[t]{0.313\textwidth}
		\scalebox{0.5}
		{
		\begin{tikzpicture}
			\begin{axis}
			[ ymin=0,ymax=0.0004,xmin=0,xmax=1,xlabel={Misclassification Error around POI},ylabel={Approximation Error}, grid = major]
			\addplot [scatter,only marks,color=blue] table [x index=0, y index =1] {"data_files_for_plotting/Text_mis_classification_error_around_POI_vs_explloss_2.dat"};
			\end{axis}
		\end{tikzpicture}
		}
		\caption{DP model $f_2$}\label{fig:approx-errorloss-2} 
	\end{subfigure}
	\begin{subfigure}[t]{0.313\textwidth}
		\scalebox{0.5}
		{
		\begin{tikzpicture}
			\begin{axis}
			[ ymin=0,ymax=0.0004,xmin=0,xmax=1,xlabel={Misclassification Error around POI},ylabel={Approximation Error}, grid = major]
			\addplot [scatter,only marks,color=blue] table [x index=0, y index =1] {"data_files_for_plotting/Text_mis_classification_error_around_POI_vs_explloss_10.dat"};
			\end{axis}
		\end{tikzpicture}
		}
		\caption{DP model $f_{10}$}\label{fig:approx-errorloss-10} 
	\end{subfigure}
        \caption{\small The approximation vs misclassification error around the POI ($50$ nearest points to the POI) for all DP models. The explanation quality deprecates as the model becomes more inconsistent around the POI with a positive correlation for all models (note the different $y$ axis max for Figures \ref{fig:approx-errorloss-2} and \ref{fig:approx-errorloss-10}). Privacy requirements result in noisier models, which results in additional privacy spending for our explanations algorithms.}\label{fig:approx_errorvsloss}
     
\end{figure}

\subsection{The Effect of Data Density}\label{sec:apiep_sparsevs_dense}
The effectiveness of Algorithm~\ref{algo:APIEP} relies, in part, on the existence of points with low gradient loss (denoted by $\beta$); such points are often close to the query point.
To evaluate this effect, we compare the performance of Algorithm~\ref{algo:APIEP} in densely and sparsely sampled queries.

We first cluster the dataset using a hierarchical clustering algorithm. 
We observe that while the Text dataset is well-distributed and sparse in the dataspace, the ACS13 dataset has over $15$ different clusters with different population densities. 
We sample in three different regimes: first, we randomly sample three samples of size $800$ from the dataset (this serves as our baseline); second, we randomly sample three samples of $800$ datapoints from the densest cluster; finally, we sample datapoints from different clusters to intentionally sparsify the sample. 
We then query each set of sampled points with Algorithm~\ref{algo:APIEP} in parallel composition with $\epsilon_{\textit{min}} = 10^{-2}$ and $T = 300$. 

Indeed, as shown in Figure \ref{fig:APIEP_diff_region}, sparse samples lead to greater expenditure of the privacy budget than dense samples. A likely explanation for this phenomenon is that densely sampled queries exhibit consistent behavior from the black-box classifier, allowing more efficient exploitation of past information by Algorithm~\ref{algo:APIEP}. That said, it may be the case that two disparate regions in the dataspace exhibit similar classifier behavior; in this scenario, Algorithm~\ref{algo:APIEP} can exploit previously computed explanations from different regions. 

\subsection{The Effects of Overfitting}\label{sec:experiments-overfitting}

Intuitively, overfitted models yield more complex decision boundaries, to achieve better training accuracy. 
According to our hypotheses, Algorithm~\ref{algo:APIEP} needs to spend more privacy budget to explain the overfitted models compare to the non-overfitted model. We train an overfitted random decision forest with $500$ trees with maximum depth $= 60$ on ACS13 dataset which achieves $98\%$ training accuracy and $80\%$ test accuracy. We first identify the overfitted data regions: we find clusters of training data that exhibit a significant difference in training accuracy between the overfitted and non-overfitted models.

We randomly sample three samples of $500$ data points from the overfitted region, and run Algorithm~\ref{algo:APIEP} on these samples in parallel with $T=300$ and $\epsilon_{\textit{min}} = 10^{-2}$ for both over-fitted and non-over-fitted models. Figure~\ref{fig:apiep-overfitting} indicates that Algorithm~\ref{algo:APIEP} requires a higher privacy budget to explain the overfitted model as compared to the non-overfitted model.

Moreover, Algorithm~\ref{algo:APIEP} needs to compute an average of $87$ queries from the sample with full computation whereas, for a non-overfitted model, $58$ queries required full computation on average. After an initial $\approx 150$ queries on average, Algorithm~\ref{algo:APIEP_para} fails to satisfy the $\sqrt{n}\sigma_{\textit{min}}\leq \beta$ (line~\ref{line_13} of Algorithm~\ref{algo:APIEP_para}) condition for both overfitted and non-overfitted models and proceeds with $T' = 95$ for most instances in all cases.

\begin{figure}[ht!]
\centering
	\begin{subfigure}[t]{0.42\textwidth}
		\scalebox{0.7}
		{
		\begin{tikzpicture}
			\begin{axis}
			[ ymin=0, ymax=0.4,xmax=3000,xlabel={Number of Queries},ylabel={Privacy Budget Spent ($\epsilon$)}, grid = major]
			\addplot[ mark=+,color=red] table [x index=0, y index =1] {"data_files_for_plotting/Adult_data_D_sparse.dat"} node[right]{\small Sparse};
			\addplot [ mark=o,color=blue] table [x index=0, y index =1]{"data_files_for_plotting/Adult_data_D_random.dat"} node[right]{\small Random};
			\addplot [ mark=x,color=black!50!green] table [x index=0, y index =1]{"data_files_for_plotting/Adult_data_D_desne.dat"} node[below right]{\small Dense};
			\end{axis}
		\end{tikzpicture}
		}
		\caption{Sparse vs. Dense Data Regions}\label{fig:APIEP_diff_region}
	\end{subfigure}
	\begin{subfigure}[t]{0.42\textwidth}
		\scalebox{0.7}
		{
		\begin{tikzpicture}
		\begin{axis}
		[ ymin=0, ymax=0.4,xmax=2000,legend style={font=\small},xlabel={Number of Queries},ylabel={Privacy Budget Spent ($\epsilon$)}, grid = major]
		\addplot [ mark=o ] table [x index=0, y index =1] {"data_files_for_plotting/Adult_data_E_over_fitted.dat"} node[above right]{\small Overfit};
		
		\addplot[ mark=x, color=blue] table [x index=0, y index =1] {"data_files_for_plotting/Adult_data_E_non.dat"}node[below right]{\small No Overfit};
		\end{axis}
		\end{tikzpicture}
		}
		\caption{Overfitting Vs. Privacy}\label{fig:apiep-overfitting}
	\end{subfigure}
	\caption{\small Figure~\ref{fig:APIEP_diff_region} shows the privacy budget spent by Algorithm~\ref{algo:APIEP} on sparse/dense/random samples from the ACS13 dataset. Figure~\ref{fig:apiep-overfitting} demonstrates the privacy budget spent by Algorithm~\ref{algo:APIEP} on an overfitted/non-overfitted model on the ACS13 dataset. The $x$-axis shows the number of queries answered; the $y$-axis is the privacy budget spent - $\epsilon$.}
\end{figure}
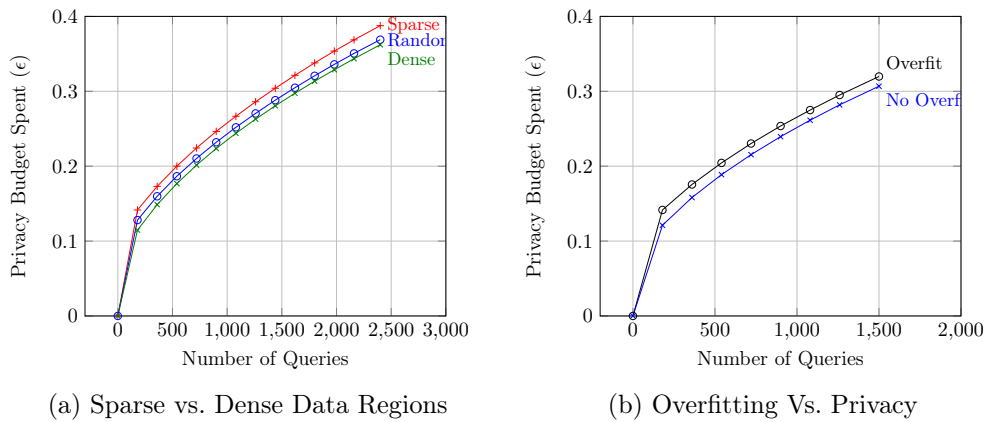

%% file: broaderimpact.tex
\section{Broader Impact}\label{sec:broaderimpact}

Transparency in AI has become an important research topic; however, as recent works show, model transparency poses potential risks to user privacy. 
Our approach provably protects against adversaries trying to recover data used in generating explanations. Several data companies are now offering model explanations as part of their ML suites\footnote{e.g. IBM \url{http://aix360.mybluemix.net/}, Microsoft \url{https://docs.microsoft.com/en-us/azure/machine-learning/how-to-machine-learning-interpretability} and Google's \url{https://cloud.google.com/explainable-ai} frameworks.}: if unsafe model explanations are offered in high stakes domains (say, in explaining users' financial or medical data as we do in our work), they offer a novel avenue for attack that has been relatively unexplored. Our methodology is easy to implement, and offers provable privacy and quality guarantees; thus, if our protocols are implemented in such services, users are guaranteed to receive good explanations that protect their personal data; the authors see this as a \coloremph{positive} impact on society. 

Our methodology does carry some risks. Our results uniformly support the existence of a tradeoff between user privacy and explanation quality. This tradeoff can seen both through our theoretical results and empirical analysis, and applies to both the privacy of the explanation dataset and the model training data. This tradeoff seems to be particularly sharp in regions where the model overfits, and for high-dimensional data. The actual effects of this tradeoff are highly implementation and data dependent, and require further analysis as to the exact risks they hold. 
It could very well be that there exist alternative modes of model explanation that offer more favorable privacy/quality tradeoffs (or indeed, no tradeoffs whatsoever); the authors believe that this is not the case. We believe that there is a general theory of model explanations and user privacy: useful model explanations leak some information, and that leakage needs to be managed from a privacy perspective. Applying our techniques does require some care: one immediate concern is \coloremph{disparate impact}. Regions where the model overfits to the data tend to suffer more from privacy/quality tradeoffs; in other words, when implementing our system on real high-stakes systems, it is important to ensure that minority groups are not adversely affected in unacceptable ways. It is undesirable to offer low-quality model explanations to minorities, but it is also undesirable to risk leaking their private information. This tradeoff has been alluded to in \cite{shokri2019privacy}, and is formally proven in our work. While it is entirely possible that this tradeoff only exists due to the shortcomings of this work, we strongly suspect that it is a universal truth: model explanations inevitably leak private information, and must undergo some degradation --- noise induced or otherwise --- before being released. This risk must certainly be assessed carefully in any future implementation of our work.

%% file: conclusions.tex
\section{Conclusions}
In this work, we provide a formal framework for studying differential privacy in model explanations. Not only are our model explanation methods differentially private, they are also approximately optimal.  Our methods are \coloremph{safe and sound}: offering both quality and privacy guarantees.  There are still many interesting open problems regarding the design of model explanations that simultaneously protect both the training and explanation datasets from exposure, while retaining explanation accuracy. 

While our work focuses on the interplay between privacy and transparency, it follows some recent insights on the effects of privacy mechanisms on model accuracy. As we demonstrate in Sections \ref{sec:apiep_sparsevs_dense} and \ref{sec:experiments-overfitting}, sparse data regions and overfitted models lead to poorer performance, either in terms of explanation accuracy, or in terms of required privacy budget. Such data regions often correspond to underrepresented population groups; our results echo those in recent works \cite{bagdasaryan2019privacyaccuracy,shokri2019privacy}: indeed, differential privacy injects more noise to model explanations offered to minority groups. Assessing the privacy/explainability tradeoffs for minority groups is a promising avenue for future exploration. 

\section*{Acknowledgments}

This work was supported by the National Research Foundation Singapore, AI Singapore award number AISG-RP-2018-009, grant number R-252-000-A20-490.

%% file: appendix.tex
\newpage 

\section{Proof of Lemma~\ref{lem:dataset-diff}}\label{appendix:prrofsensitivity}
\lemmautilitybound*
\begin{proof}
	We first note that 
	\[
	\nabla \mathcal{L}(\phi,\mathcal{X}) = \frac{2}{m}\sum_{\vec{x}\in \mathcal{X}} \alpha(\|\vec{x}-\vec{x'}\|) ( \phi^\top(\vec{x}-\vec{x'}) - f(\vec{x}) )(\vec{x}-\vec{x'})
	\]

	Now, for $\alpha(\|\vec{x}-\vec{x'}\|)\leq \frac{c}{2\|\vec{x}- \vec{x'}\|_2(\|\vec{x}- \vec{x'}\|_2 + 1)}$ and any $\vec x' \in \cal X$, $\|\nabla \mathcal{L}(\phi,\mathcal{X})  - \nabla \mathcal{L}(\phi,\mathcal{X}\setminus\{\vec x'\}) \|_2$
	\begin{align*}
	& \leq \frac{2}{m} \alpha(\|\vec x' - \vec x\|) \big((\phi^\top (\vec{x}- \vec{x}') - f(\vec{x})\big)\\
	&\leq \frac{2}{m} \alpha(\|\vec x' - \vec x\|)\big(\|\phi\|_2 \|\vec{x}- \vec{x}'\|_2 - f(\vec{x}) \big)\leq \frac{c}{m}
	\end{align*}
	all inequalities are tight,which concludes the proof.
\end{proof}

\section{Proof of Theorem~\ref{thm:uti_bound}}\label{appendix:proofofutibound}

\PIEPutilitybound*

\begin{proof}
Theorem 2 in \cite{shamir2013stochastic} shows that for a gradient descent algorithm $\phi_{t+1} \gets \Gamma_{\mathcal{C}}\Big(\phi_t - \eta(t)\hat{g}(t)\Big)$ with $\E[\hat{g}(t))] =\mathcal{L}(\phi,\mathcal{X})$, $\eta(t) = \frac{c}{\sqrt{t}}$ and $\E(\|\hat{g}(t)\|^2) \leq G^2$, 
 \begin{equation*}
     \E\big [ \mathcal{L}(\phi^\priv,\mathcal{X},c ) \big ] - \min_{\phi \in \mathcal{C}}\mathcal{L}(\phi,\mathcal{X},c ) \leq \Big ( \frac{1}{c} + cG^2  \Big) \frac{2 + \log T}{\sqrt{T}}
 \end{equation*}
  The obove bound minimize for $c = \frac{1}{G} $. By putting $C= \frac{1}{G}$, we get 
  \begin{equation*}
     \E\big [ \mathcal{L}(\phi^\priv,\mathcal{X},c ) \big ] - \min_{\phi \in \mathcal{C}}\mathcal{L}(\phi,\mathcal{X},c ) \leq  2G \times \frac{2 + \log T}{\sqrt{T}}
 \end{equation*} 
  
  In Algorithm~\ref{Algorithm-1}, $\hat{g}(t) = \nabla \mathcal{L}(\phi,\mathcal{X}) + \vec Z$, where $Z_i\sim \mathcal{N}(0,\sigma^2\mathrm{I})$. As $\E(\vec Z) = \vec 0$, we can write $\E(\|\hat{g}(t)\|^2)$ as
  \begin{align*}
      \E(\|\hat{g}(t)\|^2) &= \| \nabla \mathcal{L}(\phi,\mathcal{X}) \|^2 + \E\Big [\sum_{i=1}^n Z_i^2\Big]\\
      &\leq 1 + \frac{16Tn\log (\euler + \sqrt T\epsilon / \delta )\log \frac{T}{\delta}}{m^2\epsilon^2} 
  \end{align*}
 
For $T \leq \frac{m^2 \epsilon^2}{32n \log^2 (\euler + \epsilon / \delta )} $, $G \leq \sqrt{2}$ which implies that $\cal E(\phi^\priv ) \in \log \bigg( \frac{\log T}{\sqrt{T}} \bigg)$

Also for $T > \frac{m^2 \epsilon^2}{32n \log^2 (\euler + \epsilon / \delta )}  $; function $2G \times \frac{2 + \log T}{\sqrt{T}}$ increases as $T$ increases. Therefore by setting $T=\frac{m^2 \epsilon^2}{32n \log^2 (\euler + \epsilon / \delta )} $ for the optimal upper bound, we get
  \begin{equation*}
     \cal E(\phi^\priv ) \in \mathcal{O}\bigg( \frac{\sqrt{n}\log \big(\frac{1}{\delta} \big) \log (m\epsilon/\sqrt{n})}{m \epsilon} \bigg)
 \end{equation*}
 
 For large $m$, ignoring the $\log \frac{1}{\delta}$ term,
   \begin{equation*}
     \cal E(\phi^\priv ) \in \mathcal{O}\bigg( \frac{\sqrt{n} \log (m/\sqrt{n})}{m \epsilon} \bigg)
 \end{equation*}

\end{proof}

\section{Proof of Lemma \ref{lem:alpha_closness}}\label{appendix:proof_alpha_closeness}
\lemmaalphacloseness*

\begin{proof}
    We first note that  when $\|\vec x - \vec z\| \le r$ $\frac{\alpha(\|\vec x - \vec v\|)}{\alpha(\|\vec x - \vec z\|)} = \frac{\alpha(\|\vec x - \vec v\|)}{1}\le 1$ since $\alpha(\cdot)\le 1$ on any input. 
Thus, we analyse the case when $\|\vec v - \vec z\| > r$.
If $\|\vec x - \vec v \| \leq r$ then:
\begin{align*}
    &\frac{\alpha(\|\vec x - \vec v\|)}{\alpha(\|\vec x - \vec z\|)} =\frac{1}{\frac{c}{2\|\vec x - \vec z\|(1+\|\vec x - \vec z\|)}}\\
    &= \frac{2\|\vec x - z\|(1 + \|\vec x - \vec z \|)}{c}\le \frac{2(r+d)(1 + r+d)}{c}\\
    &=\frac{2(r^2 + r ) + 2(d^2 + 2rd)}{c}= 1 + \frac{2(d^2 + 2rd)}{c}
\end{align*}

Now, for $\|\vec x - \vec v\| > r $: if $\|\vec x - \vec v\| \geq \|\vec x - \vec z \|$ then by definition of the weight function, $\frac{\alpha(\|\vec x - \vec v\|)}{\alpha(\|\vec x - \vec z\|)} \leq 1$. Therefore we analyse the case when $\|\vec x - \vec v\| < \|\vec x - \vec z \|$; 
\begin{equation*}
    \frac{\alpha(\|\vec x - \vec v\|)}{\alpha(\|\vec x - \vec z\|)} = \frac{(\|\vec x - z\|)(1 + \|\vec x - \vec z \| )}{(\|\vec x - \vec v\|)(1 + \|\vec x - \vec v \| )}
\end{equation*}
As we know, $\|\vec x - \vec z\| \leq d$; $\| \vec x - \vec z \| \leq d + \|\vec x - \vec v\| $. By using this inequality; 

\begin{align*}
    &\frac{\alpha(\|\vec x - \vec v\|)}{\alpha(\|\vec x - \vec v\|)} \leq \frac{(d + \|\vec x - \vec v\|)(1 + \|\vec x - \vec v \| + d)}{(\|\vec x - \vec v\|)(1 + \|\vec x - \vec v \| )}\\
    &\leq \left( 1 + \frac{d}{r} \right) \left( 1+ \frac{d}{r + 1} \right)
    \leq 1 + \frac{d^2 + 2rd}{r^2}
\end{align*}
Which shows that for all $\vec x \in \cal X$ 
\begin{equation*}
 \frac{\alpha(\|\vec x - \vec v\|)}{\alpha(\|\vec x - \vec z\|)} \leq  1 + \max \left(\frac{(d^2 + 2rd)}{r^2}, \frac{2(d^2+2rd+d)}{c}\right)
 \end{equation*}
Finally, suppose that $d \in o(r)$; in this case, $\frac{(d^2 + 2rd)}{r^2} < d\times \frac{3}{r}  \in \cal O(d)$, and 
\begin{align*}
\frac{2(d^2+2rd+d)}{c} & < \frac{d(3r+1)}{c}\in \cal O(d),
\end{align*}
which shows that, for $d\in o(r)$
\begin{equation}
    \frac{\alpha(\|\vec x - \vec v\|)}{\alpha(\|\vec x - \vec z\|)} \leq 1 + \cal O(d) \quad , \forall \vec x \in \cal X\label{eqn:alpha_closeness}
\end{equation}
\end{proof}

\section{Proof of Theorem~\ref{thm:adpative_closness}}\label{appendix:proofofclosness}
Next, let us prove Theorem \ref{thm:adpative_closness}.
\thmadaptivecloseness*

\begin{proof}
	First, $\cal L (\phi^*(\vec z),\vec v,\cal X)$ equals:
	\begin{align*}
	&\frac{1}{m} \sum_{\vec x \in \cal X} \alpha(\|\vec x - \vec v\|) (\phi^*(\vec z)^\top (\vec x - \vec v) - f(\vec x))^2=\\
	&\frac{1}{m} \sum_{\vec x \in \cal X} \alpha(\|\vec x - \vec v\|) (\phi^*(\vec z)^\top (\vec x - \vec z) - f(\vec x)+ \phi^*(\vec z)^\top(\vec z -\vec v))^2  \\
	&\le\frac{1}{m} \sum_{\vec x \in \cal X} \alpha(\|\vec x - \vec v\|) (\phi^*(\vec z)^\top (\vec v - \vec z) - f(\vec x))^2 + d^2\\
	&\qquad \qquad + \frac{d}{m} \sum_{y\in \cal X} \alpha(\|\vec x - \vec v\|) |\phi^*(\vec z)^\top (\vec x - \vec z) - f(\vec x) |
	\end{align*}
	
	Now, for $\|\vec v - \vec z\|\leq d < r$; 
	$$
	\alpha(\|\vec x - \vec v\|) |\phi^*(\vec z)^\top (\vec x - \vec z) - f(\vec x) | \leq (1+d+r)$$ 
	and 
	$\alpha(\|\vec x - \vec v\|)\leq 1$. By Lemma~\ref{lem:alpha_closness}, $\alpha(\|\vec x - \vec v\|) \leq (1+\cal O(d))\alpha(\|\vec x - \vec z\|)$, which implies that $\cal L(\phi^*(\vec z),\vec v,\cal X)$ is upper bounded by 
	
	\begin{align*}
	\frac{1+\cal O(d)}{m}\sum_{\vec x \in \cal X} \alpha(\|\vec x- \vec z\|) (\phi^*(\vec z)^\top (\vec x - \vec z) - f(\vec x))^2+ (2+r)d + \cal O(d^2)&\leq \\
	\frac{1+\cal O(d)}{m}\sum_{\vec x \in \cal X} \alpha(\|\vec x- \vec z\|) (\phi^*(\vec v)^\top (\vec x - \vec v) - f(\vec x)\phi^*(\vec v)^\top(\vec v - \vec z) )^2+ (2+r)d + \cal O(d^2)&
	\end{align*}
	Invoking Lemma~\ref{lem:alpha_closness} again we have $\alpha(\|\vec x - \vec z\|) \leq (1+\cal O(d))\alpha(\|\vec x - \vec v\|)$, and the last inequality follows from the fact that $\cal L(\phi^*(\vec z),\vec z,\cal X) \leq \cal L(\phi^*(\vec v),\vec z,\cal X) $. By combining these results, $\cal L(\phi^*(\vec z),\vec v,\cal X)$ is upper bounded by
	\begin{equation*}
	\leq (1+\cal O(d))^2 \cal L(\phi^*(\vec v),\vec v,\cal X)) + 2(2+r)d + \cal O(d^2)
	\end{equation*}
	Now, $\cal L(\phi^*(\vec v),\vec v,\cal X))\leq \cal L(\vec 0,\vec v,\cal X))\leq 1 $, Which shows that, 
	\begin{align} 
	\cal L(\phi^*(\vec z), \vec v,\cal X) - \cal L(\phi^*(\vec v), \vec v,\cal X) \leq (6+2r)d + \cal O(d^2) \label{eqn:similar_loss}
	\end{align}
	and concludes the proof of the first part. For the second part, take any $\phi \in \cal C$, then $|\cal L(\phi, \vec v,\cal X) - \cal L(\phi, \vec z, \cal X)|$
	
	\begin{align*}
	&\leq \frac{\cal  O(d^2)}{m} \sum_{\vec x \in \cal X} \alpha(\|\vec x - \vec v\|) (\|\vec x - \vec v\| + \|\vec x - \vec z\|+2)
	\end{align*}
	Now, if $\|\vec x - \vec v\| \leq r$ then $\alpha(\|\vec x - \vec v\|) (\|\vec x - \vec v\| + \|\vec x - \vec z\|+2) \leq (2(r+2) + d)$ and if $\|\vec x - \vec v\| > r$ then $\alpha(\|\vec x - \vec v\|) (\|\vec x - \vec v\| + \|\vec x - \vec z\|+2) \leq \left(1 + \frac{r+d+1}{r}\right)$, which shows that, for all $\phi \in \cal C$;
	\begin{equation*}
	|\cal L(\phi, \vec v,\cal X) - \cal L(\phi, \vec z, \cal X)| \in \cal O(d^2)
	\end{equation*}
	
\end{proof}
\section{Proof of Lemma~\ref{lem:small_grad}}\label{appendix:proofofintUB}

\lemmasmallgrad*
\begin{proof}
The proof of the lemma uses \cite[Theorem 2]{shamir2013stochastic}. 
We denote $g_t = \eta(t)\big[\nabla \mathcal{L}(\phi^{\{t\}},\vec v,\mathcal{X}) + \mathcal{N}(0,\sigma^2\mathrm{I})\big] $ and $\mathcal{L}(\phi^{\{t\}},\vec v,\mathcal{X})$ by $\cal L(\phi^{\{t\}})$. By extending the inequality in \cite[Equation (2)]{shamir2013stochastic} we obtain;
\begin{equation*}
\E[g_t^\top (\phi^{\{t\}} - \phi^{\{T-k\}})] \leq \frac{1}{2c}[\sqrt{T} - \sqrt{T-k}] + \frac{1}{2} \sum_{t=T-k}^T \frac{c\E[\|g_t\|^2]}{\sqrt t}
\end{equation*}

Now, we know that $\|\nabla \cal L (\phi^{\{0\}})\|^2 = \beta^2$. By using Taylor's Theorem we get,
\begin{align*}
  &\nabla \cal L (\phi^{\{t+1\}})  = \nabla \cal  L (\phi^{\{t\}}) + \nabla^2 \cal L (\phi^{\{t\}})(\phi^{\{t+1\}} - \phi^{\{t\}})\\
  & = (\mathrm{I} - \eta(t)\nabla^2 \cal L (\phi^{\{t\}})) \nabla \cal L (\phi^{\{t\}}) - \eta(t)\nabla^2 \cal L (\phi^{\{t\}})Z_t
\end{align*}
Where $Z_t \sim \mathcal{N}(0,\sigma^2\mathrm{I}) $. As $\nabla^2 \cal L (\phi^{\{t\}})$ is a semi-positive-definite matrix with $\|\nabla^2 \cal L (\phi^{\{t\}})\| \leq 1$, 

\begin{align*}
  \E[\|\nabla \cal L (\phi^{\{t+1\}})\|^2]  \leq  \E[\|\nabla \cal L (\phi^{\{t\}})\|^2] + \eta^2(t)n\sigma^2&\le\\
  \E[\|\nabla \cal L (\phi_{0})\|^2] + n\sigma^2 \sum_{l = 1}^t \frac{c^2}{l} \leq \beta^2 + n\sigma^2c^2 \log t\\
  \leq \beta^2 + n\sigma^2c^2 \log T&
\end{align*}

Therefore $\E[\|g_t\|^2]\leq (\beta^2 + n\sigma^2) + n\sigma^2c^2 \log T$. As $\cal L(\cdot)$ is a convex function, by \cite[Theorem 2]{shamir2013stochastic}, we obtain,
\begin{equation*}
    \cal E (\phi^{\{T\}},\vec x) \leq \left(\frac{1}{c} + c(n\sigma^2+ \beta^2) + c^3n\sigma^2 \log T \right) \frac{(\log T +2)}{\sqrt T}
\end{equation*}

The above bound is minimized for the positive solution of:
$-\frac{1}{c^2} + (n\sigma^2 + \beta^2) + 3n\sigma^2c^2 = 0$.
Which implies that the above bound minimizes for,
\begin{equation}
    c^2 = \frac{2}{\left((n\sigma^2 + \beta^2) + \sqrt{(n\sigma^2 + \beta^2)^2 + 12n\sigma^2 \log T }  \right)}\label{eqn:learning_rate}
\end{equation}

Plugging back the optimal $c$, we get 
\begin{equation*}
    \cal E(\phi^{\{T\}},\vec x) \in \cal O  \left( \left( (\sigma^2n + \beta^2 )^\frac{1}{2} + (n\sigma^2 \log T)^\frac{1}{4}  \right) \frac{\log T}{\sqrt{T}} \right)
\end{equation*}
Which concludes the proof.
\end{proof}

\section{Proof of Theorem~\ref{thm:small_grad_priv}} \label{appendix:fasterconvergence}

Let us now prove Theorem \ref{thm:small_grad_priv}. 
\coronumberofiteration*
\begin{proof}
Let $b = \max(\sqrt{n}\sigma, \beta) $, substituting $T' = b T $ into Lemma~\ref{lem:small_grad}, we get,
\begin{equation*}
    \cal E( \phi^{\{T\}},\vec v) \in \cal O \left( \frac{b^\frac{1}{2} \log^{\frac{5}{4}} (b T)}{b^{\frac{1}{2} - \frac{1}{4a}} T^{\frac{1}{2}}} \right)
\end{equation*}
For $b \leq \frac{1}{(\log T)^a}$, $b^\frac{1}{4a}\log^{\frac{5}{4}} (b T) \leq \log T$, which implies that $\cal E ( \phi^{\{T\}},\vec v) \in \cal O \left (\frac{\log T}{T^{\frac{1}{2}}} \right)$,
concluding the proof.
\end{proof}

\section{Proof of Theorem~\ref{thm:uti_bound_adaptive}}\label{appendix:proof_of_theorem_uti_bound_adaptive}

\thmutiboundadaptive*

\begin{proof}
Theorem~\ref{thm:uti_bound_adaptive} is a corollary of Theorem~\ref{thm:adpative_closness}, Theorem~\ref{thm:small_grad_priv} and Theorem~\ref{thm:uti_bound}. First, if the {\bf If} condition in \textit{line}~\ref{line-5} of Algorithm \ref{algo:APIEP} is satisfied, then as $l<k$, by Theorem~\ref{thm:adpative_closness},
\begin{align*}
  &|\cal L(\phi^\priv(\vec z_j),\vec z_j) - \cal L(\phi^*(\vec z_j),\vec z_j) |\leq |\cal L(\phi^\priv(\vec z_j),\vec z_j) \\
  &- \cal L(\phi^\priv(\vec z_j),\vec x) | + |L(\phi^\priv(\vec z_j),\vec x) - \cal L(\phi^*(\vec z_j),\vec z_j) |\\
  & \in \cal O(\log^2 T/T) + \cal O(\log T/\sqrt T) \in \cal O \left(\frac{\log T}{\sqrt{T}} \right)
\end{align*}
For other cases; by Theorem~\ref{thm:uti_bound} and Theorem~\ref{thm:small_grad_priv};
\begin{equation*}
    \cal E(\phi^\priv(\vec z_j),\vec z_j) \in \cal O \left(\frac{\log T}{\sqrt{T}} \right) 
\end{equation*}
By substituting $T = \sqrt m$, we get that $\cal E (\phi^\priv(z_j),\vec z_j) \in \cal O \left( \log m/ m^{\frac{1}{4}}  \right)$, which concludes the proof.
\end{proof}

\section{Proof of Theorem~\ref{thm:dptrain}} \label{appendix:proofofdptrain}
Before we prove Theorem~\ref{thm:dptrain}, we prove the following technical lemma. 

\begin{lemma}\label{lem:tvupperbound}
If $X,Y$ and $Z$ are random variables defined on the same probability space with same compact range $\mathcal R$ and $X\independent Z$, $Y\independent Z$ and $X\independent Y$ then the total variation distance,

\begin{equation*}
\mathrm{TV}(X+Z,Y+Z) \leq \max_{x,y} \mathrm{TV}(X+Z|X=x , Y+Z| Y=y)  
\end{equation*}
\end{lemma}

\begin{proof}
The definition of the Total variation distance between two distribution is defined as
\begin{align}
    &\mathrm{TV}(X+Z,Y+Z) = \sup_E (\Pr[X+Z \in E] - \Pr[Y+Z \in E])\notag\\ 
                        &= \sup_E \left(\int_{\mathcal R}\Pr[x+Z \in E] f_X(x) dx - \int_{\mathcal R} \Pr[y+Z \in E] f_Y(y) dy\right )\notag\\
                    &\leq \sup_E \left( \max_x \Pr[x+Z \in E] \int_{\mathcal R} f_X(x) dx - \min_y \Pr[y+Z \in E] \int_{\mathcal R}  f_Y(y) dy\right )\notag\\
                    &= \sup_E \left( \max_x \Pr[x+Z \in E]  - \min_y \Pr[y+Z \in E] \right ) \label{eqn:tvbound}
\end{align}
Where $f_X$ and $f_Y$ denotes the density function of the probability distribution of $X$ and $Y$. Now if $X=x^*$ and $Y=y^*$ maximize the difference $\max_{x,y} (\Pr[Z+x\in E] - \Pr[Z+y \in E])$ for any $E$, that implies that $x^*=\argmax \Pr[Z+x\in E]$ and $y^*=\argmin \Pr[Z+y\in E]$. The equality in Equation~\ref{eqn:tvbound} follows from the fact that $f_X$ and $f_Y$ are probability densities. This shows that 
\begin{align*}
    \mathrm{TV}(X+Z,Y+Z) &\leq \sup_E \left(\max_{x,y}\left (\Pr[x + Z\in E] - \Pr[y + Z\in E]\right)  \right)\\
    &= \max_{x,y} \left(\sup_E\left (\Pr[x + Z\in E] - \Pr[y + Z\in E]\right) \right)\\
    &= \max_{x,y} \mathrm{TV}(X+Z|X=x , Y+Z| Y=y)
\end{align*}
This concludes the proof.
\end{proof}

\thmtrainprivimprove*
\begin{proof}
We first bound the sensitivity of gradients with respect to the training dataset; i.e. if a particular datapoint is removed from the training dataset of the black-box model $f$, then the change in the gradient of the loss function is bounded given $\phi$ and explanation dataset $\mathcal X$. Suppose that $f$ and $f'$ are black-box models trained on the training set $\mathcal T$ and $\mathcal T\setminus \vec v$ respectively; then for any $\phi_1,\phi_2 \in \mathcal C_{2,1}$, explanation dataset $\mathcal X$ and weight function $\alpha(\cdot) \in \mathcal F(c,\vec z)$:

\begin{align}
    &\| \nabla \mathcal L(\phi_1,\mathcal X,f) - \nabla \mathcal{L} (\phi_2,\mathcal X,f') \|\notag\\
    &= \frac{2}{m}\Big\| \sum_{\vec y \in \mathcal X}\alpha(\|\vec y - \vec z\|) ((\phi_1 - \phi_2)^\top (\vec y - \vec x) + f(\vec y)-f'(\vec y))(\vec y - \vec z)  \Big\| \notag \\
    &\leq \frac{c}{m} \sum_{\vec y \in \mathcal X} \frac{(|f(\vec y)-f'(\vec y)|) + (\|\phi_1\|+ \|\phi_2\|)\|\vec y - \vec z\|}{(1+\|\vec y - \vec z\|)}\leq 2c \label{eqn:trainsensitivity}
\end{align}
The first inequality is derived from the fact that $\alpha(\cdot) \in \cal F(c,\vec z)$. Let $\mathcal M(\mathcal T)$ and  be the training process on the dataset $\mathcal T$ which outputs a probability distribution on the parameter space of the black-box model $\Theta$ where each $\theta$ corresponds to the learned black-box function. For a fixed parameter $\theta\in \Theta$ (black-box function $f$), our explanation algorithm computes $T$ many noisy iterative gradients $\mathcal L(\vec x, \phi^{\{i\}},f)+Z$ for $i=0,\dots,T-1$ and output $\phi^{\{T\}}$, a probability distribution over explanations (set $\mathcal C$). 
To show that this transformation\footnote{More formally we can define the transformation as a Markov kernel. For more details, see \cite{balle2019privacy}} improves the privacy guarantees of the training dataset due to randomness added during the optimization, we use similar techniques described in \cite{balle2019privacy}, and utilize contraction transformations.

Now, for any two different parameterizations $\theta_1, \theta_2 \in \Theta$, we denote the corresponding black-box functions $f_1$, $f_2$. For simplicity, we denote the distribution of $\phi^{\{T-1\}} - \eta({T-1})\nabla \mathcal L(\phi^{\{T-1\}},f)$ as $\mu$. The total variation between the distribution $\phi^{\{T\}}|\theta_1,\phi^{\{T\}}|\theta_2$ can be written as:

\begin{align}
  &\mathrm{TV}(\phi^{\{T\}}|\theta_1,\phi^{\{T\}}|\theta_2) = \mathrm{TV}\left (\mu +\eta(T-1)Z_{T-1})|\theta_1, \mu + \eta(T-1)Z_{T-1}|\theta_2) \right ); Z\sim \mathcal N(o,\sigma^2\mathrm I) \notag \\
&\leq \max_{\phi_1,\phi_2\in \mathcal C} \mathrm{TV}\left (\mu +Z_{T-1})|\theta_1\land \phi^{\{T-1\}} = \phi_1 , \mu +Z_{T-1}|\theta_2\land \phi^{\{T-1\}} = \phi_2) \right ) \label{eqn:tvonlastite}
\end{align}

Equation~\ref{eqn:tvonlastite} follows from the Lemma~\ref{lem:tvupperbound}. Once we condition on $\phi^{\{T-1\}}$, both distributions in the above equation follows a Normal distribution with different means; we write $\phi_1 = \mu\mid \phi^{\{T-1\}}$ and $\phi_2 = \mu\mid \phi^{\{T-1\}}$ respectively. Both have the same variance: $\eta(T-1)^2\sigma^2\mathrm I$. Note that $\phi_1$ and $\phi_2$ are not random quantities. For Gaussian random variables with same variance matrix and different means, \cite[Theorem-1]{barsov1987estimates}, provides the exact value of the total variation distance between them. Using this result we obtain:
\begin{align}
  &\mathrm{TV}(\phi^{\{T\}}|\theta_1,\phi^{\{T\}}|\theta_2)\\
  &\leq \max_{\phi_1,\phi_2\in \mathcal C} 2\Pr\left[ Z \in\left( 0, \frac{\|\phi_1 - \phi_2 + \eta(T-1)(\nabla\mathcal L(\phi_2,f_1) -\nabla\mathcal L(\phi_2,f_2) \|}{2\eta(T-1)\sigma}\right)  \right ] \notag \\
  &\leq \max_{\phi_1,\phi_2\in \mathcal C} 2\Pr\left[ Z \in\left( 0, \frac{\|\phi_1\| +\| \phi_2\| + \eta(T-1)\|\nabla\mathcal L(\phi_2,f_1) -\nabla\mathcal L(\phi_2,f_2) \|}{2\eta(T-1)\sigma}\right)  \right ]\notag \\
  &\leq \max_{\phi_1,\phi_2\in \mathcal C} 2\Pr\left[ Z \in\left( 0, \frac{ 1 + \eta(T-1)c}{\eta(T-1)\sigma}\right)  \right ]  = 2\Pr\left[ Z \in\left( 0, \frac{1 + \eta(T-1)c}{\eta(T-1)\sigma}\right)  \right ] \notag\\
  &\leq 2\Pr\left[ Z \in\left( 0, \frac{m\epsilon}{16\sqrt 2 \log \left(\frac{2T}{\delta}\right)}\right)  \right ]\label{eqn:trainprivimprproof}
\end{align}
Where $Z\sim \mathcal N(0,1)$. The first inequality in Equation~\ref{eqn:trainprivimprproof} follows from Equation~\ref{eqn:trainsensitivity}. Equation~\ref{eqn:trainprivimprproof} and \cite[Theorem~1]{balle2019privacy} concludes the proof of the first part of the theorem.
\newline
The second part follows from Equation~\ref{eqn:trainsensitivity}. The sensitivity of the gradient wrt. training dataset is $2m$ times the sensitivity of the explanation dataset. Concluding the proof.
\end{proof}